\pdfoutput=1
\documentclass[10pt,logo,copyright]{nvidiatechreport}
\usepackage{algorithm}
\usepackage[most]{tcolorbox}

\usepackage{parskip}        %
\usepackage{amsfonts}       %
\usepackage{nicefrac}       %
\usepackage{animate}        %
\usepackage{subcaption}
\usepackage{tabularx}
\usepackage{makecell}
\usepackage{adjustbox}
\usepackage{setspace}
\usepackage{siunitx}        %
\usepackage{enumitem}       %
\newcolumntype{M}[1]{>{\centering\arraybackslash}m{#1}}
\usepackage{float}

\usepackage{tikz}
\usetikzlibrary{positioning,shapes,arrows}
\usepackage{amsmath,amsfonts,bm, bbm,leftindex}
\usepackage{multirow}
\usepackage{comment}
\usepackage{gensymb}
\usepackage{lipsum}
\usetikzlibrary{arrows.meta, positioning, fit}
\usepackage[para]{threeparttable}
\usepackage{tikz}
\usetikzlibrary{tikzmark}
\usetikzlibrary{decorations.pathreplacing}

\let\cite\citep
\usepackage[numbers]{natbib}

\definecolor{darkred}{rgb}{0.7, 0.0, 0.0}

\usepackage{pifont}
\usepackage{wrapfig}

\renewcommand*{\backref}[1]{}
\renewcommand*{\backrefalt}[4]{}

\usepackage[nameinlink]{cleveref}
\crefname{equation}{Eq.}{Eqs.}
\crefname{figure}{Fig.}{Figs.}
\crefname{section}{Sec.}{Sec.}
\crefname{appendix}{App.}{App.}
\crefname{table}{Tab.}{Tabs.}
\crefname{algorithm}{Algo}{Algo}
\crefname{thm}{Thm}{Thm}
\Crefname{thm}{Thm}{Thm}
\crefname{prop}{Prop}{Prop}
\usepackage{longtable}
\usepackage{mathtools}
\usepackage{siunitx}

\usepackage{ragged2e}
\definecolor{nvidiagreen}{HTML}{76B900}
\definecolor{bestrow}{HTML}{E1EBD7}
\newcolumntype{Y}{>{\raggedleft\arraybackslash}X}
\newcolumntype{L}[1]{>{\raggedright\arraybackslash}p{#1}}

\usepackage{tikz}
\usepackage{siunitx}
\usepackage{algorithm}
\usepackage{algpseudocode}
\usepackage{comment}
\definecolor{nvidiaGreen}{RGB}{118,185,0}
\usepackage{appendix}
\usepackage{titletoc}
\usepackage{wrapfig}
\Crefname{figure}{Fig.}{Figs.}
\Crefname{table}{Tab.}{Tab.}

\usetikzlibrary{positioning, calc, shapes, arrows.meta, backgrounds, spy}

\usepackage{amsthm}
\newtheorem{proposition}{Proposition}
\newtheorem{remark}{Remark}

\theoremstyle{definition}
\newtheorem{definition}{Definition}

\newcommand{\worldtrace}{\textsc{WorldTrace}\xspace}
\newcommand{\loopbench}{\textsc{LoopBench}\xspace}
\newcommand{\wtfield}{\textsc{WorldTrace-Field}\xspace}
\newcommand{\wtlandmark}{\textsc{WorldTrace-Landmark}\xspace}
\newcommand{\wtunified}{\textsc{WorldTrace}\xspace}

\makeatletter
\renewcommand{\section}{%
  \@startsection{section}{1}{\z@}%
                {-1.0ex \@plus -0.4ex \@minus -0.2ex}%
                { 0.6ex \@plus  0.2ex \@minus  0.1ex}%
                {\large\bf\raggedright}%
}
\renewcommand{\subsection}{%
  \@startsection{subsection}{2}{\z@}%
                {-0.8ex \@plus -0.4ex \@minus -0.2ex}%
                { 0.3ex \@plus  0.2ex}%
                {\normalsize\bf\raggedright}%
}
\renewcommand{\subsubsection}{%
  \@startsection{subsubsection}{3}{\z@}%
                {-0.6ex \@plus -0.4ex \@minus -0.2ex}%
                { 0.15ex \@plus 0.2ex}%
                {\normalsize\bf\raggedright}%
}
\renewcommand{\paragraph}{%
  \@startsection{paragraph}{4}{\z@}%
                {0.5ex \@plus 0.4ex \@minus 0.2ex}%
                {-1em}%
                {\normalsize\bf}%
}

\@ifundefined{@toptitlebar}{}{%
\renewcommand{\@toptitlebar}{%
  \hrule height 4\p@%
  \vskip 0.20in%
  \vskip -\parskip
}
\renewcommand{\@bottomtitlebar}{%
  \vskip 0.22in%
  \vskip -\parskip
  \hrule height 1\p@%
  \vskip 0.04in
}
\renewcommand{\@maketitle}{%
  \vbox{%
    \hsize\textwidth
    \linewidth\hsize
    \vskip 0.03in%
    \@toptitlebar
    \centering
    {\LARGE\bf \@title\par}
    \@bottomtitlebar
    \if@anonymous
      \begin{tabular}[t]{c}\bf\rule{\z@}{24\p@}
        Anonymous Author(s) \\
        Affiliation \\
        Address \\
        \texttt{email} \\
      \end{tabular}%
    \else
      \def\And{%
        \end{tabular}\hfil\linebreak[0]\hfil%
        \begin{tabular}[t]{c}\bf\rule{\z@}{24\p@}\ignorespaces%
      }%
      \def\AND{%
        \end{tabular}\hfil\linebreak[4]\hfil%
        \begin{tabular}[t]{c}\bf\rule{\z@}{24\p@}\ignorespaces%
      }%
      \begin{tabular}[t]{c}\bf\rule{\z@}{24\p@}\@author\end{tabular}%
    \fi
    \vskip 0.12in \@minus 0.1in%
  }%
}
\renewenvironment{abstract}%
{%
  \vskip 0.02in%
  \centerline{\large\bf Abstract}%
  \vspace{0.3ex}%
  \begin{quote}%
}{%
  \par%
  \end{quote}%
  \vskip 0.3ex%
}
}%
\makeatother


\newcommand{\Sec}[1]{Sec.~\ref{#1}}
\newcommand{\App}[1]{App.~\ref{#1}}
\newcommand{\Tab}[1]{Tab.~\ref{#1}}
\newcommand{\Fig}[1]{Fig.~\ref{#1}}
\newcommand{\Alg}[1]{Alg.~\ref{#1}}
\newcommand{\Def}[1]{Def.~\ref{#1}}

\newcommand{\Prop}[1]{Prop.~\ref{#1}}
\newcommand{\Eq}[1]{Eq.~\eqref{#1}}

\definecolor{revblue}{HTML}{1D4ED8}
\newif\ifshowrev
\showrevtrue

\newcommand{\timeSym}{t}                            %
\newcommand{\keySym}{K}                             %
\newcommand{\querySym}{Q}                           %
\newcommand{\dimCount}{c}                           %

\newcommand{\axisTemp}{t}                           %
\newcommand{\axisH}{h}                              %
\newcommand{\axisW}{w}                              %

\newcommand{\virtualMark}{v}                        %
\newcommand{\labelTrain}{\mathrm{train}}            %

\newcommand{\freqSup}{^{\idxFreq}}                  %

\newcommand{\realpart}{\mathrm{Re}}                 %
\newcommand{\eul}{e}                                %
\newcommand{\imagi}{i}                              %

\newcommand{\numChunks}{N}                          %
\newcommand{\numSummarySlots}{\smash{N_s}}          %
\newcommand{\numRecentSlots}{\smash{N_r}}           %
\newcommand{\localAttnSize}{\smash{L_{\mathrm{attn}}}} %
\newcommand{\trainLen}{L_{\labelTrain}}             %
\newcommand{\framesPerBlock}{F}                     %
\newcommand{\numSourceFrames}{M}                    %
\newcommand{\slotIdx}{s}                            %
\newcommand{\recentCache}{\mathcal{R}}              %
\newcommand{\summaryCache}{\mathcal{S}}             %

\newcommand{\idxFreq}{f}                            %
\newcommand{\idxSrc}{m}                             %

\newcommand{\qpos}{q}                               %
\newcommand{\kpos}{k}                               %
\newcommand{\tminv}{\smash{\timeSym_{\min}^{\virtualMark}}}    %
\newcommand{\tmaxv}{\smash{\timeSym_{\max}^{\virtualMark}}}    %
\newcommand{\tv}{\timeSym^{\virtualMark}}                       %
\newcommand{\tvslot}[1][\slotIdx]{\smash{\timeSym_{#1}^{\virtualMark}}}  %
\newcommand{\landmarkTime}{\smash{\timeSym_{\ell^{*}}}}       %

\newcommand{\ropeBase}{\theta}                                      %
\newcommand{\ropefreq}{\smash{\ropeBase_{\idxFreq}}}                %
\newcommand{\numTempPairs}{\smash{\dimCount_{\axisTemp}}}           %
\newcommand{\numSpatialPairsH}{\smash{\dimCount_{\axisH}}}          %
\newcommand{\numSpatialPairsW}{\smash{\dimCount_{\axisW}}}          %
\newcommand{\Rot}[1]{R\!\left(#1\right)}                            %
\newcommand{\Krot}[1][\idxSrc]{\smash{\keySym_{\timeSym_{#1}}\freqSup}}                  %
\newcommand{\Kcx}[1][\idxSrc]{\smash{\keySym_{#1}\freqSup}}                     %
\newcommand{\Kcxmean}{\smash{\bar{\keySym}\freqSup}}                          %
\newcommand{\Knaive}{\smash{\bar{\keySym}_{\mathrm{naive}}\freqSup}}                      %
\newcommand{\Kfield}[1][\tv]{{\keySym}_{\mathrm{field}}\freqSup\!\left(#1\right)}         %
\newcommand{\Klandslot}[1][\tv]{{\keySym}_{\mathrm{land}}\freqSup\!\left(#1\right)}      %
\newcommand{\Klandsrc}{\smash{\keySym_{\landmarkTime}\freqSup}}                          %
\newcommand{\Qrope}{\smash{\querySym_{\qpos}\freqSup}}                                   %
\newcommand{\relOffset}{\smash{\delta_{\qpos,\kpos}}}                                    %
\newcommand{\trainOffset}{\smash{\Delta\timeSym_{\labelTrain}}}                          %
\newcommand{\attnLogit}{\smash{a_{\qpos,\kpos}\freqSup}}                              %
\newcommand{\contentTerm}{\smash{A_{\qpos,\kpos}\freqSup}}                               %

\newcommand{\sbThreshold}{\tau}                     %

\newcommand{\loopK}{K}                              %
\newcommand{\loopR}{R}                              %

\newcommand{\Reals}{\mathbb{R}}                     %
\newcommand{\bigO}[1]{O\!\left(#1\right)}           %
\newcommand{\simplex}[1]{\Delta_{#1}}               %
\newcommand{\identMat}[1]{I_{#1}}                   %
\newcommand{\valueSym}{V}                           %
\newcommand{\queryVec}{q}                           %
\newcommand{\softmaxOp}{\operatorname{softmax}}     %
\newcommand{\attnOp}{\mathcal{A}}                   %
\newcommand{\projAttnOp}{\widehat{\mathcal{A}}}     %
\newcommand{\fullAttnWeights}{\smash{\alpha_{\queryVec}}}      %
\newcommand{\protoAttnWeights}{\smash{\hat\alpha_{\queryVec}}} %
\newcommand{\projMat}{P}                            %
\newcommand{\numPastFrames}{T}                      %
\newcommand{\querySet}{\mathcal{Q}}                 %
\newcommand{\stackedFull}{\mathbf{A}}              %
\newcommand{\stackedProto}{\mathbf{B}}             %
\newcommand{\resid}[1][\projMat]{\smash{r_{\queryVec}(#1)}}    %

\title{Addressable Memory for Video World Models}

\begin{document}

\author{%
  Xindi Wu$^{1\,2}$
  Sven Elflein$^{1\,3\,4}$
  James Lucas$^{1}$
  Olga Russakovsky$^{2}$
  Laura Leal-Taix\'{e}$^{1}$
  Despoina Paschalidou$^{1}$
  Jonathan Lorraine$^{1}$
  Aljosa Osep$^{1}$\\
  \small $^{1}$NVIDIA \quad
         $^{2}$Princeton University \quad
         $^{3}$University of Toronto \quad
         $^{4}$Vector Institute\\
  \small \href{https://research.nvidia.com/labs/sil/projects/WorldTrace/}{\textcolor{nvidiaGreen}{https://research.nvidia.com/labs/sil/projects/WorldTrace/}}%
}

\maketitle
\vspace{-0.01\textheight}
\hypertarget{abstract}{}

\begin{abstract}
We study visual persistence in interactive video world models. 
These models rely on a Key-Value (KV) cache as a growing visual memory to carry forward previously generated frames. 
However, we find that models can no longer reliably address stored content once rollouts extend beyond the training horizon, because temporal Rotary Positional Embedding (RoPE) offsets can exceed the training range and the model struggles to retrieve the relevant visual information through attention.
Moreover, naively compressing the cache in the RoPE-rotated space degrades memory by averaging keys with incompatible positional phases.
To address this, we propose \worldtrace, a training-free memory framework for long-horizon visual persistence. 
\worldtrace keeps compressed memory addressable by assigning each summary slot a distinct, in-distribution virtual position.
Within this cache, \worldtrace uses the model's own attention to decide which observations to retain verbatim and which redundant ones to merge into compressed summaries, all within one bounded addressable memory bank. 
We evaluate this ability with \loopbench, a benchmark that tests whether a compressed cache can reconstruct a previously visited scene after a long detour.
On \loopbench, \worldtrace improves average DINOv2 recall similarity by 32.6\% over a sliding window at the same cache budget, extending visually persistent generation without retraining.

\end{abstract}

\abscontent
\section{Introduction}

Memory is what makes a generated world persistent.
A world model should not only predict what comes next, but also preserve what has already been seen.
World models aim to simulate worlds we can explore~\citep{ha2018world,hafner2025dreamerv3,genie3_2025}, and recent autoregressive video world models~\citep{lingbot-world,matrixgame2025,oasis2024} pursue this vision by generating video scenes chunk-by-chunk, with each chunk attending over a Key-Value (KV) cache of prior context.
These models promise interactive applications such as next-generation game engines~\citep{valevski2025gamengen,oasis2024,matrixgame2025} and closed-loop robot simulators~\citep{agarwal2026cosmos,basant2026nvidia}, in which users can move freely through a visual world and revisit prior locations.
Such revisits demand \emph{visual persistence}, meaning that when an agent returns, the generated results should be consistent with the scene's original appearance, not a plausible-looking alternative.

In practice, visual persistence degrades rapidly once generation exceeds the training context length. A natural way to address this issue is to compress the linearly growing KV cache into a fixed-size memory~\citep{zhang2023h2o,li2024snapkv, cai2025pyramidkv,kim2026memrope} equal to the size of the context seen during training.
However, we find that the bottleneck is not simply whether past content is stored in the cache, but whether it remains addressable and how it is compressed.
As the generation horizon extends beyond the training context length, the positional encoding these models rely on (e.g., RoPE~\citep{su2024rope}) is queried at position indices outside its training-bounded range, so the model needs to attend over relative distances that are out-of-distribution (OOD), degrading attention-based memory retrieval.
In other words, even if past memories are stored in the KV cache, the model cannot reliably retrieve them.
A summary whose offset exceeds the trained range can therefore become harder to reach.
This reveals that long-horizon failure is fundamentally a problem of \textit{addressability}, a prerequisite for long-horizon visual persistence.

\paragraph{\worldtrace.}
To tackle this problem, we propose \worldtrace, a training-free approach for addressable compressed memory over long horizons.
We assign each summary slot a fixed temporal offset from the current frame within the range seen during training, so that content from distant frames remains accessible to attention. The cache is organized into a recent window and summary slots for the distant horizon.
Within this cache, we formulate memory compression as a structured projection of the growing visual history into a fixed set of retrievable traces. The projection can average redundant visual content into compact summaries while preserving observations with distinctive details verbatim. Both types of memory fit within the same cache and share a position assignment that keeps them addressable as generation continues.

Our experiments show that addressable memory extends the effective memory horizon. To evaluate episodic recall under controlled scene revisits, we propose \loopbench, a \emph{memory} benchmark that tests whether the compressed KV cache can retrieve a previously visited scene after a long detour. While the sliding-window baseline begins to forget scene structure after only a few seconds and loses structural fidelity when revisiting the same location, our \worldtrace maintains consistent scene reconstruction over long rollouts and recovers the original scene appearance (\Fig{fig:teaser}).

\paragraph{Our contributions are:}
\begin{enumerate}[leftmargin=1.5em,itemsep=2pt,topsep=2pt,parsep=0pt]
\item We identify two coupled challenges for long-horizon generation in autoregressive video world models, \emph{where} memory is placed and \emph{what} it stores. Temporal offsets beyond the training range can impair retrieval from the cache, and RoPE phase cancellation makes naively compressed summaries uninformative (\Sec{sec:rethinking}).
\item We propose \worldtrace, a training-free framework that keeps compressed memory both addressable and informative. 
It assigns each summary a virtual position whose offset stays within the training range (\Sec{subsec:worldtrace_positions}) and organizes the growing visual history into a fixed set of retrievable traces. 
The same cache supports temporal-group averages through \wtfield (\Sec{subsec:wtfield}) and retained scene-entry frames through \wtlandmark (\Sec{subsec:wtlandmark}). %
\item We introduce \loopbench, a \emph{memory} benchmark of controlled scene revisits that measures whether a compressed cache can reconstruct a previously visited scene (\Sec{sec:experiments}). On Matrix-Game 2, \worldtrace improves average DINOv2 recall similarity across the four \loopbench task families by 32.6\% without retraining over the default sliding window baseline.
\end{enumerate}

\begin{figure}[t]
  \centering
  \includegraphics[width=\linewidth]{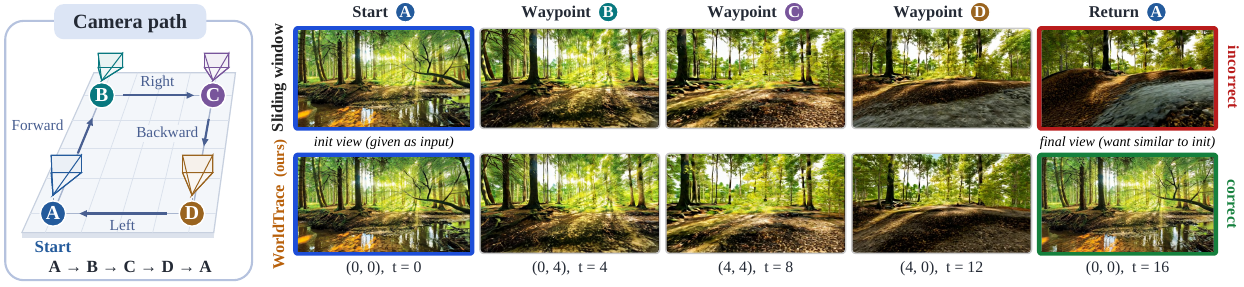}
  \caption{\textbf{Addressable memory maintains long-horizon visual persistence.} \emph{(Left)} Loop topology A$\to$B$\to$C$\to$D$\to$A. \emph{(Right)} Representative frames on Matrix-Game-2 along the path, covering the initial scene~A (blue border), waypoints B, C, and D, and the return to~A. \emph{Sliding window} (top) mismatches the reference appearance at the return (red border). Our \worldtrace with frozen landmark keys (green border) matches scene~A at the return. }
  \label{fig:teaser}
\end{figure}

\section{Rethinking Memory in Video World Models}
\label{sec:rethinking}

Autoregressive video world models retain prior context as cached keys and values during generation. Long-horizon memory therefore depends both on whether relevant past context remains inside the cache and on whether the model can effectively attend to it during future generation. 
Long-horizon memory requires both retrieving retained content and preserving it during compression. Retrieval can become less reliable when a cached key's temporal offset exceeds the training range (\Sec{sec:rope_ood}). 
Compression introduces a separate problem. Averaging keys with different RoPE rotations can partially cancel their components and change the query-key scores used to retrieve the corresponding values (\Sec{sec:phase_cancel_failure}).
We summarize the notation in \App{app:notation}.

\subsection{Position Determines Whether Memory Is Addressable}
\label{sec:rope_ood}
Attention retrieves information from cached values using weights computed from query–key scores.
With the cached key’s temporal position held fixed, its relative temporal offset from the query grows as later frames are generated. The offset can exceed the trained range, which can reduce the reliability of attention-based retrieval. Concretely, during training the query only attends over a bounded range of relative offsets $\relOffset$ between the current frame and each cached frame:
\[
0 \leq \relOffset = \qpos - \kpos \leq \trainOffset,
\]
where $\qpos$ is the absolute position of the current query frame, $\kpos$ that of a cached key frame, with $\kpos \leq \qpos$ under causal autoregressive attention, and $\trainOffset$ the largest offset seen during training. For temporal RoPE pair $\idxFreq$ at angular frequency $\ropefreq$, its contribution to the query-key attention score is:
\[
\attnLogit = \realpart\!\left(\contentTerm \eul^{\imagi\ropefreq\relOffset}\right),
\]
where $\realpart(\cdot)$ is the real part and $\contentTerm$ is the content-dependent inner product in canonical coordinates between the query at $\qpos$ and the key at $\kpos$.
During long-horizon inference, these offsets grow ever larger, and $\relOffset$ can exceed $\trainOffset$. The affected RoPE components then take on rotation angles the model never saw during training.
Each temporal RoPE pair rotates at a rate determined by its frequency. High-frequency pairs may complete multiple rotations within the training window, whereas low-frequency pairs may cover only a limited range of phases. At longer temporal offsets, these low-frequency pairs can encounter phases outside that range. The effect on retrieval depends on both the relative phase and the content-dependent query-key interactions. 
We report the rotation angle at each temporal frequency in \App{app:rope_background}.

Several concurrent methods investigate these positional issues.
Infinity-RoPE~\citep{infinityrope2026} proposes Block-relative RoPE, which reindexes each cached frame relative to the current AR chunk so that $\relOffset$ always stays within the trained range, keeping generation in-distribution. It targets infinite-length video generation rather than memory compression. Applied to a compressed cache, however, this reindexing clips every offset beyond the trained range to $\trainOffset$, so summary slots that stand for older history share one temporal position and lose their distinct positional cues.
MemRoPE~\citep{kim2026memrope} compresses history into dual EMA memory streams with unrotated key storage and contiguous online indexing, a design that avoids this collapse for its two summary positions but does not extend to an $N$-slot cache, where capped Block-relative positions would again collide.
Our \worldtrace targets an $N$-slot compressed cache by assigning every slot a distinct, in-distribution slot-rank position, so any number of summaries stay individually addressable at any horizon.

\subsection{Content Compression Determines Whether Memory Is Informative}
\label{sec:phase_cancel_failure}

In an interactive world model, compressed memory needs to support reconstructing a previously visited place after it has left the current view.
During causal generation, however, attention weights are computed for the current query, so a past frame that receives little attention now may still be needed once the camera returns to the place it shows.
Compression therefore needs to reduce redundancy among repeated observations while preserving distinctions that may become useful again under a different viewing condition.
A summary at an in-distribution temporal position still needs to preserve the distinctions between the source views.
A common way to compress cached keys is averaging~\citep{bolya2023tome,rae2020compressive}.
With RoPE-rotated keys, the averaged key at frequency $\idxFreq$ is:
\[
\Knaive = \smash{\frac{1}{\numSourceFrames}}\sum\nolimits_{\idxSrc=1}^{\numSourceFrames} \Rot{\ropefreq \timeSym_{\idxSrc}}\,\Kcx,
\]
where $\numSourceFrames$ is the number of source frames and $\timeSym_{\idxSrc}$ is the timestamp of source frame $\idxSrc$. $\Kcx$ denotes that frame's canonical key at frequency $\idxFreq$, and $\Rot{\ropefreq \timeSym_{\idxSrc}}$ is its temporal RoPE rotation.
Each frame is rotated by a different angle $\ropefreq \timeSym_{\idxSrc}$ before averaging, so $\Knaive$ sums vectors that point in different directions.
Nearby timestamps produce similar rotations when their phase difference is small. Averaging then introduces little additional cancellation from the positional rotations.
If they come from distant timestamps, their angles can point in opposite directions and partially cancel, weakening that frequency's contribution to the summary regardless of what it encoded.
Naive averaging in RoPE-rotated space can therefore distort the compressed memory.

\paragraph{The two failures are coupled.}
Most prior methods fix only one of the two, with one line of work adjusting \emph{where} memories are placed, reparameterizing or capping offsets (\eg Block-relative~\citep{infinityrope2026}), while another decides \emph{what} each summary stores, whether by smoothing past keys with dual-rate EMA~\citep{kim2026memrope}, merging them by temporal averaging~\citep{bolya2023tome,rae2020compressive}, or evicting entries~\citep{zhang2023h2o,li2024snapkv,cai2025pyramidkv}. Addressable positions cannot help if compression cancels the signal a summary should carry, and an informative summary cannot help if it sits at an out-of-distribution position. Maintaining visual persistence over long rollouts therefore requires compressed memory that stays both addressable and informative beyond the training window.

\section{\worldtrace}
\label{sec:worldtrace}

We propose \worldtrace, a training-free approach for addressable memory in video world models. \worldtrace maintains a bounded trace of distant history in fixed summary slots via canonical-space key assignment, while keeping each slot \emph{addressable} (\Sec{subsec:worldtrace_positions}). We formulate this cached memory compression for visual persistence as a structured attention approximation (\Sec{subsec:sparse_overview}). Within this slot structure, \wtfield (\Sec{subsec:wtfield}) averages groups of frames to cover longer history and reduce redundancy, while \wtlandmark (\Sec{subsec:wtlandmark}) retains selected frames for episodic recall, and \worldtrace combines both in a single memory bank (\Sec{subsec:combined}).

\subsection{Virtual Position Assignment with \worldtrace Slot Indexing}
\label{subsec:worldtrace_positions}

\paragraph{Cache structure.}
\label{subsec:arch}

\begin{figure}[t]
\centering
\includegraphics[width=\linewidth]{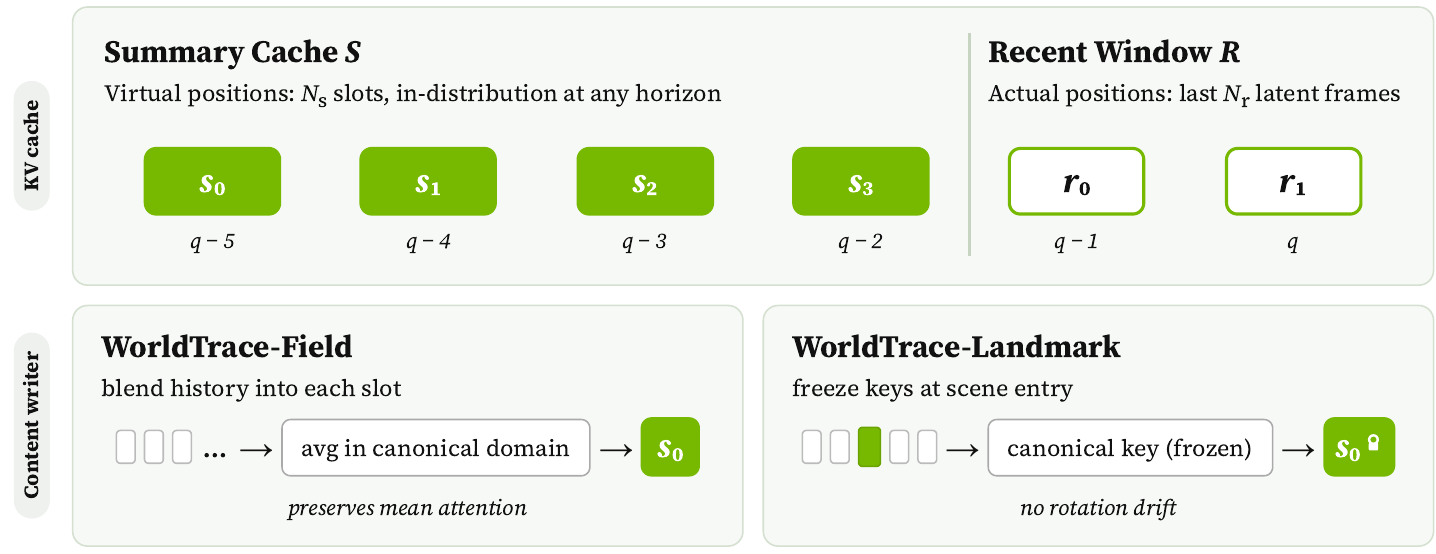}
\caption{\textbf{\worldtrace overview.} \worldtrace partitions the context window into a compressed \emph{summary} cache $\summaryCache$ for the distant past and a verbatim \emph{recent} window $\recentCache$ (top). Each slot occupies one latent frame in the cache, with recent slots holding a verbatim frame and summary slots holding compressed history at a virtual position. Both \wtfield (bottom left), which averages unrotated keys, and \wtlandmark (bottom right), which freezes the canonical keys of scene-entry landmarks, store keys in canonical space and re-rotate them at each slot's virtual position.}
\label{fig:worldtrace_pipeline}
\end{figure}

As shown in \cref{fig:worldtrace_pipeline}, we partition the local attention window $\localAttnSize$ into a \textbf{recent window} $\recentCache$ of $\numRecentSlots$ verbatim latent frames and a \textbf{summary cache} $\summaryCache$ of $\numSummarySlots$ slots storing compressed history, so that $\numSummarySlots + \numRecentSlots = \localAttnSize$.
We assign each summary slot a distinct temporal offset from the current query. Keeping these offsets fixed as generation advances keeps them within the trained range at every horizon. Slot rank determines the offset through the linear assignment below.
Consider a full cache, after at least $\localAttnSize$ frames have been generated, and choose $\localAttnSize{-}1\leq\trainOffset$. The recent window occupies the latest $\numRecentSlots$ positions. A summary position needs to precede that window while staying no more than $\trainOffset$ steps behind the current query, so it satisfies:
\begin{equation}
\tminv = \max\bigl(0,\ \qpos - \trainOffset\bigr), \qquad \tmaxv = \qpos - \numRecentSlots.
\label{eq:tv_bounds}
\end{equation}
We place the summary slots consecutively within this interval, ordered from oldest to newest.
\begin{definition}[\worldtrace Slot Indexing]
\label{def:vpos}
Given current query position $\qpos$, local attention window $\localAttnSize$ ($=\numSummarySlots + \numRecentSlots$), and summary cache size $\numSummarySlots$, the virtual position of summary slot $\slotIdx$ is:
\begin{equation}
\tvslot \;=\; \qpos - (\localAttnSize{-}1{-}\slotIdx), \quad \slotIdx = 0,\ldots,\numSummarySlots{-}1.
\label{eq:vpos}
\end{equation}
\end{definition}
The oldest summary slot has offset $\localAttnSize{-}1$, and the newest has offset $\numRecentSlots$. Both endpoints lie within the trained range under the cache-size condition above. The relative offsets depend on slot rank, not the absolute horizon.

\paragraph{Why canonical keys.}
A slot's virtual position $\tvslot$ is set by slot rank and generally differs from its source frames' timestamps. We therefore remove each source key's temporal RoPE rotation before storing or averaging it, and apply the rotation at $\tvslot$ when the key is read. Without this step, (i)~keys with distinct rotations $\Rot{\ropefreq \timeSym_{\idxSrc}}$ would partially cancel when averaged, and (ii)~the compressed key would carry the phase of its source timestamps rather than that of $\tvslot$. MemRoPE~\citep{kim2026memrope} also stores keys canonically but uses contiguous online indexing, whereas we pair canonical storage with slot-rank positions (\Def{def:vpos}). Both \worldtrace writers share this canonical key domain and differ only in which keys they write.
The content writer determines which canonical keys contribute to each slot.

\subsection{Memory Compression as Structured Sparse Attention}
\label{subsec:sparse_overview}

With slot positions fixed, the remaining choice is \emph{what} to store in each slot.
Storing every past frame makes the cache grow with generation length. Evicting all distant frames bounds its size but removes the cached evidence for later revisits. 
We use a cache of $\localAttnSize$ slots to approximate full attention over all $\numPastFrames$ past frames. We consider two structured projections of history: sparse selection of landmark frames and averaging within temporal groups. The projection matrix $P \in \mathbb{R}^{\localAttnSize \times \numPastFrames}$ defines this pattern by giving slot $i$ the compressed key $P_i\Kcx$, rotated to the slot's virtual position. The recent window uses verbatim identity rows for the last $\numRecentSlots$ frames, and the $\numSummarySlots$ summary rows carry the compressed information.

\paragraph{Approximating full attention.}
Attending over all $\numPastFrames$ past frames would produce a softmax-weighted combination of their values, with weights $\alpha_q$ determined by query-key similarity. In order to keep the cache bounded, memory compression replaces this with projected attention over $\localAttnSize$ compressed keys, where row $i$ of $P$ defines which past frames contribute to compressed key $i$, given by $P_i\Kcx$. Re-encoding each compressed key at $\tvslot$ keeps the approximation in-distribution at any horizon. We approximate the error with the \emph{distribution mismatch} between the full attention weights $\alpha_q$ and the weights induced by the compressed cache (\Prop{prop:approx_bound}, \App{app:sparse_attention}). The choice of $P$ therefore determines what history is preserved under a fixed slot budget.

\paragraph{Choosing the projection matrix.}
This view reduces memory compression design to the choice of $P$, with a good $P$ minimizing the distribution mismatch above for the queries encountered during rollout. Relaxing the entries of $P$ from fixed patterns to arbitrary nonnegative values turns this minimization into a nonnegative matrix factorization (NMF) problem~\citep{lee1999learning}, approximately factoring the matrix of full attention weights into the product of two nonnegative low-rank factors (\App{app:sparse_attention}). The optimal row structure of $P$ then depends on these queries. Rows that average contiguous temporal groups suit future attention that spreads smoothly over history, while rows that select single frames verbatim suit attention concentrated on a few salient past frames. These two attention patterns motivate the two writers we introduce next, \wtfield and \wtlandmark, which together constitute the unified \worldtrace memory (\Sec{subsec:combined}).

\subsection{Canonical Key Averaging with \wtfield}
\label{subsec:wtfield}

\wtfield uses each summary row of $P$ to average a contiguous temporal group of frames. This represents older history at a coarser temporal resolution while preserving a verbatim recent window.
We average in canonical coordinates to avoid cancellation from the source frames' different RoPE rotations.
For each temporal pair $\idxFreq$, the compressed key at virtual position $\tv$ is:
\begin{equation}
\Kfield
\;=\; \Rot{\ropefreq \tv} \, \smash{\frac{1}{\numSourceFrames}}\sum\nolimits_{\idxSrc=1}^{\numSourceFrames} \Rot{-\ropefreq \timeSym_{\idxSrc}} \, \Krot,
\label{eq:rdc}
\end{equation}
we unrotate each key to its canonical content $\Rot{-\ropefreq \timeSym_{\idxSrc}}\,\Krot$, average in that space, and then re-encode at virtual position $\tv$. RoPE rotates queries and keys but leaves values unchanged. Slot values are the average of their source-frame values. Here $\numPastFrames$ is the number of latent frames seen so far. We split the $\numPastFrames{-}\numRecentSlots$ frames outside the recent window into $\numSummarySlots$ contiguous groups, ordered from oldest to newest. Each slot therefore represents approximately $\numSourceFrames \approx (\numPastFrames{-}\numRecentSlots)/\numSummarySlots$ source frames. This count is determined by the grouping and grows with rollout length.

Each new chunk adds $\framesPerBlock$ frames, displacing older frames from the recent window and changing which frames each group covers, so the slot averages have to be maintained as the rollout grows. The writer keeps the canonical key of each evicted frame in host memory. When group boundaries change, it recomputes each slot mean from the original source keys assigned to that group. 
\begin{remark}[Mean score preservation]
\label{rem:mean_attn}
A summary slot represents $\numSourceFrames$ source frames. Canonical averaging preserves their mean query-key score when all source keys are assigned the same virtual position $\tv$. It removes the source keys' positional phases before averaging, avoiding the phase cancellation that can distort scores when rotated keys are averaged (\Sec{sec:phase_cancel_failure}). By linearity of the dot product, the score of the average key equals the average of the source-key scores at that shared position. Attention weights are different. Softmax exponentiates each score and normalizes across keys, so averaging several keys generally changes their total weight. 

\end{remark}

\subsection{Frozen Landmark Keys with \wtlandmark}
\label{subsec:wtlandmark}

To optimize for episodic recall, \wtlandmark fills the $\numSummarySlots$ slots with verbatim high-value past frames rather than averaged summaries. Landmark Attention inserts landmark tokens in trained transformers~\citep{mohtashami2023landmark}. Here, we select existing scene-entry frames and store their canonical keys at slot-rank virtual positions without retraining.

\paragraph{Landmark selection.}
For each incoming frame, we compute the cosine distance between its canonical keys and those of the preceding frame. A distance above threshold $\sbThreshold$ marks a scene-entry event. \wtlandmark fills the $\numSummarySlots$ summary slots with the most recent scene-entry frames verbatim, and while fewer than $\numSummarySlots$ scene-entry events have been detected, the remaining slots repeat the oldest available landmark. Once more than $\numSummarySlots$ scene-entry events have been detected over a long rollout, the oldest is evicted to make room for the newest.
Each retained landmark frame is assigned its summary slot's virtual position using \Eq{eq:vpos}. Its original timestamp is used only to remove the source RoPE rotation.

\paragraph{Frozen landmark keys.}
Slots are ordered from oldest ($\slotIdx{=}0$) to newest. Once the landmark bank is full, inserting a new scene-entry frame evicts the oldest landmark and shifts the remaining landmarks one slot toward the older end. Between insertions, slot ranks stay fixed, while virtual positions change as the query timestamp advances. Repeatedly unrotating a stored key and re-rotating it at an updated virtual position can accumulate floating-point error, particularly in bfloat16 precision~\citep{wang2025anchor}.
\wtlandmark removes this drift by freezing each selected key in canonical form, storing that key once at landmark time and applying a single fresh rotation to the current virtual position whenever the key is attended to:
\begin{equation}
\Klandslot[\tvslot]
\;=\; \Rot{\ropefreq \tvslot} \, \Rot{-\ropefreq \landmarkTime} \, \Klandsrc,
\label{eq:landmark}
\end{equation}
where $\landmarkTime$ is the original timestamp of the selected landmark frame. This is the single-frame case of \Eq{eq:rdc}, with $\numSourceFrames{=}1$, so the stored key is retained without averaging and rotated to its current virtual position when read. The canonical key $\Rot{-\ropefreq \landmarkTime}\,\Klandsrc$ is computed once at landmark time and reused across subsequent reads. The canonical-caching mechanism is shared with concurrent work that stores RoPE keys canonically and rotates at attention time (\eg MemRoPE~\citep{kim2026memrope}). \wtlandmark differs in (i) applying it selectively at detected scene-entry events rather than to every cached key, and (ii) combining it with the slot-rank position assignment, so each frozen landmark occupies a distinct in-distribution virtual position at any horizon.

\subsection{Combining Retained Frames and Summaries}
\label{subsec:combined}

\paragraph{Memory compression under a fixed budget.}
A fixed memory budget limits how many observations we can retain. Merging similar observations makes room for a longer history, but can remove details needed to recall a specific scene. We therefore evaluate each candidate merge over a range of possible future queries, using the attention-response error it introduces, defined as the difference between each query's attention output before and after the merge, together with a penalty for losing unique information.
\wtlandmark preserves selected observations as individual key-value pairs for episodic recall. \wtfield instead combines compatible observations into averaged summaries, weighting each summary by the number of observations it represents. This retains a longer history within the same budget. These summaries reduce storage and attention costs compared with keeping every observation.

The model can retrieve a previous view without an explicit camera trajectory because its attention query depends on the current hidden state. This state can reflect visual context, motion, and action conditioning. For example, as the camera turns back toward a garage, the query can match keys from an earlier view of that garage and retrieve the associated values. Successful recall depends on preserving the features that distinguish the garage from other scenes.

\paragraph{Attention-guided compression.}
\worldtrace combines individual observations and compressed summaries in a single bounded memory bank. The model reads both through the same attention mechanism. To decide which slots to merge or evict, we evaluate candidate updates using attention-response error on sampled queries. We select the update that minimizes this error, so that compression alters what the model retrieves as little as possible. Compatible observations are averaged, with each summary weighted by the number of observations it represents. For incompatible observations, we retain a paired exemplar.
Under this rule, an observation retained without merging forms a landmark, while a group of averaged observations forms a field summary. \wtlandmark and \wtfield correspond to these two cases. Both use slot indexing to assign each memory slot a distinct in-distribution virtual position. Whether a slot holds an averaged key or a frozen landmark key, we store it in canonical form and re-encode it at the slot's virtual position. This allows the memory update to select content independently of its original position. 

\section{Experiments}
\label{sec:experiments}
We evaluate whether \worldtrace helps a video world model reconstruct a previously visited scene after a detour. \Sec{subsec:setup} introduces the models, \loopbench protocol, metrics, and baselines. \Sec{sec:loop} reports recall across four task families. We compare \wtfield, \wtlandmark, and the combined \wtunified (which uses both) at the same attended-cache size in \Tab{tab:ablation}, with total memory use reported separately.

\subsection{Setup}
\label{subsec:setup}

\paragraph{Models.}
Our main evaluation uses Matrix-Game-2 (MG2-1.3B)~\citep{matrixgame2025}, a distilled 1.3B-parameter autoregressive game world model based on Wan 1.3B T2V~\citep{wan2025} with 3D-RoPE. Each AR chunk contains $\framesPerBlock{=}3$ latent frames. The training-time KV-cache extent is $\trainLen{}=6$ chunks, while the local attention window spans $2$ chunks, or $6$ latent frames. The corresponding maximum temporal offset is $\trainOffset=5$ latent-frame steps.
We provide additional experiments on the 14B LingBot-World model~\citep{lingbot-world} in \App{app:extra_results}.

\paragraph{Metrics.}
We compare the frames generated at the revisited location (\eg A in ABA) on the first visit and on the return, using two recall metrics. Scene Consistency ($\uparrow$) measures the cosine similarity of DINOv2 features~\citep{oquab2024dinov2}, and Return LPIPS ($\downarrow$) measures a learned perceptual distance~\citep{zhang2018perceptual}. These measure scene-feature similarity and perceptual appearance, respectively. Neither alone establishes exact reconstruction. We also report Attention Mismatch ($\Delta_{\mathrm{attn}}$ $\downarrow$), a compression diagnostic that compares attention distributions under identical queries and source history. We provide the full definitions and preprocessing in \App{app:metrics}.

\paragraph{Baselines.}
Our baseline is MG2's default sliding-window cache. It evicts the oldest entries without compressing them and uses no additional position correction. Retaining only recent entries already bounds their relative offsets. We compare it with the \wtfield and \wtlandmark writers and with \wtunified, which combines both. The evaluated \wtunified implementation keeps a bank of verbatim frames alongside compressed summaries and selects an entry using the model's attention queries (\Tab{tab:ablation}).
All methods in \Tab{tab:loopbench} share the same attended-cache size. Their total memory use differs. We compare with the training-free positional baseline MemRoPE~\citep{kim2026memrope}.

\paragraph{Benchmark.}
We introduce \loopbench, a \emph{memory} benchmark for scene revisits in AR world models. The model follows a path through waypoints and returns to a previously visited location. We evaluate the return frame against the earlier generated view at a geometrically matched location, without an external reference image. This tests whether a model preserves a particular scene across a detour. A plausible new scene can still differ from the one previously generated. Varying the detour separates the effects of elapsed rollout length and intervening views.
\loopbench varies waypoint count ($\loopK$, \Fig{fig:loops}), rollout length ($\numChunks$), camera orientation, and multi-revisit patterns. \Fig{fig:loopbench_gallery} in \App{app:loopbench} shows all evaluated configurations.

\begin{figure}[t]
\centering
\definecolor{loopboxblue}{HTML}{A9B8E0}
\definecolor{looptabblue}{HTML}{CBD7F2}
\tikzset{
  nd/.style={draw, circle, fill=gray!20, minimum size=0.50cm, inner sep=0pt},
  loopout/.style={->, blue!70},
  loopret/.style={->, blue!70, dashed},
  panel/.style={>=Stealth, thick, font=\scriptsize\rmfamily},
  panelbox/.style={draw=loopboxblue, line width=1.2pt, rounded corners=10pt,
                   inner sep=0pt, minimum width=3.4cm, minimum height=3.4cm},
  paneltab/.style={fill=looptabblue, rounded corners=3.5pt, inner xsep=8pt,
                   inner ysep=3.5pt, font=\footnotesize\rmfamily, text=black}}

\begin{tikzpicture}[baseline=(p.center)]
  \begin{scope}[panel, local bounding box=bb]
    \node[nd] (a) at (0.8,0)   {A};
    \node[nd] (b) at (0.8,2.1) {B};
    \draw[loopout, bend left=24] (a) to node[right]{\scriptsize $\times\!8$} (b);
    \draw[loopret, bend left=24] (b) to (a);
  \end{scope}
  \node[panelbox] (p) at (bb.center) {};
  \node[paneltab] at (p.north) {\textbf{ABA}\enspace $\numChunks{=}16$};
\end{tikzpicture}%
\hspace{0.6cm}%
\begin{tikzpicture}[baseline=(p.center)]
  \begin{scope}[panel, local bounding box=bb]
    \node[nd] (a) at (0,0)       {A};
    \node[nd] (b) at (0,1.90)    {B};
    \node[nd] (c) at (1.85,1.90) {C};
    \draw[loopout] (a) -- node[left,  xshift=1pt ]{\scriptsize $\times\!5$} (b);
    \draw[loopout] (b) -- node[above, yshift=-1pt]{\scriptsize $\times\!5$} (c);
    \draw[loopret] (c) -- node[right, xshift=-1pt]{\scriptsize $\times\!7$} (a);
  \end{scope}
  \node[panelbox] (p) at (bb.center) {};
  \node[paneltab] at (p.north) {\textbf{ABCA}\enspace $\numChunks{=}17$};
\end{tikzpicture}%
\hspace{0.6cm}%
\begin{tikzpicture}[baseline=(p.center)]
  \begin{scope}[panel, local bounding box=bb]
    \node[nd] (a) at (0,0)       {A};
    \node[nd] (b) at (0,1.90)    {B};
    \node[nd] (c) at (1.90,1.90) {C};
    \node[nd] (d) at (1.90,0)    {D};
    \draw[loopout] (a) -- node[left ]{\scriptsize $\times\!4$} (b);
    \draw[loopout] (b) -- node[above]{\scriptsize $\times\!4$} (c);
    \draw[loopout] (c) -- node[right]{\scriptsize $\times\!4$} (d);
    \draw[loopret] (d) -- node[below]{\scriptsize $\times\!4$} (a);
  \end{scope}
  \node[panelbox] (p) at (bb.center) {};
  \node[paneltab] at (p.north) {\textbf{ABCDA}\enspace $\numChunks{=}16$};
\end{tikzpicture}
\caption{\textbf{Three \loopbench geometries.} We propose \loopbench, a memory benchmark with diverse geometries, including ABA (straight reversal), ABCA (approximate L-triangle with diagonal $5\sqrt{2} \approx 7$), and ABCDA (square). $\numChunks$ is rollout length in AR chunks. Each path returns to scene~A. We provide a full gallery of rollout-length, orientation, and multi-revisit settings in \App{app:loopbench} (\Fig{fig:loopbench_gallery}).}
\label{fig:loops}
\end{figure}

\begin{figure}[t!]
  \centering
  \includegraphics[width=\linewidth]{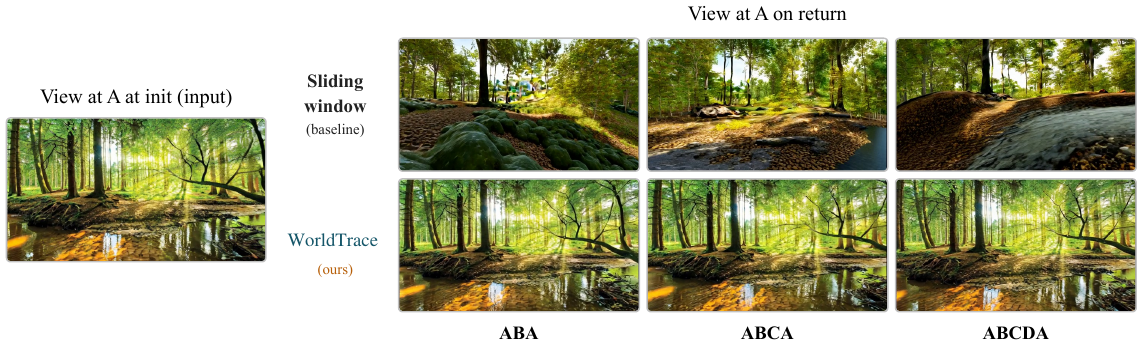}
  \caption{\textbf{\loopbench qualitative results} on the three revisit geometries in \Fig{fig:loops} (ABA, ABCA, and ABCDA). The sliding-window baseline fails to reconstruct scene~A when revisiting. \worldtrace matches scene~A in all three, keeping memory addressable across waypoints.}
  \label{fig:landmark_qualitative}
\end{figure}

\begin{table}[t]
\centering
\caption{\textbf{Verbatim, addressable memory improves episodic recall in every \loopbench tier.} DINOv2 ($\uparrow$) / LPIPS ($\downarrow$) between the frames generated at the revisited location (\eg A in ABA) on the first visit and on the return, mean over $n{=}51$ rollouts, six-frame attended window for every method ($\numSummarySlots{=}4$, $\numRecentSlots{=}2$ for the WorldTrace variants). The \worldtrace row reports the combined variant that uses both \wtfield and \wtlandmark. Loop geometries are in \Fig{fig:loopbench_gallery} (\App{app:loopbench}).}
\label{tab:loopbench}
\small
\setlength{\tabcolsep}{4pt}
\resizebox{\columnwidth}{!}{%
\begin{tabular}{l ccc}
\toprule
\textit{Tier 1, topology} & ABA ($\loopK{=}1$, $\numChunks{=}16$) & ABCA ($\loopK{=}2$, $\numChunks{=}17$) & ABCDA ($\loopK{=}3$, $\numChunks{=}16$) \\
\cmidrule(l){2-4}
Sliding window         & $0.577$ / $0.643$ & $0.627$ / $0.643$ & $0.599$ / $0.640$ \\
MemRoPE                & $0.623$ / $0.621$ & $0.630$ / $0.630$ & $0.607$ / $0.626$ \\
\textbf{\worldtrace}   & $\mathbf{0.798}$ / $\mathbf{0.512}$ & $\mathbf{0.767}$ / $\mathbf{0.548}$ & $\mathbf{0.805}$ / $\mathbf{0.489}$ \\
\midrule
\textit{Tier 2, rollout length} & ABA short ($\numChunks{=}8$) & ABA ($\numChunks{=}16$) & ABA long ($\numChunks{=}32$) \\
\cmidrule(l){2-4}
Sliding window         & $0.805$ / $0.525$ & $0.577$ / $0.643$ & $0.422$ / $0.700$ \\
MemRoPE                & $0.816$ / $0.505$ & $0.623$ / $0.621$ & $0.464$ / $0.694$ \\
\textbf{\worldtrace}   & $\mathbf{0.891}$ / $\mathbf{0.460}$ & $\mathbf{0.798}$ / $\mathbf{0.512}$ & $\mathbf{0.750}$ / $\mathbf{0.537}$ \\
\midrule
\textit{Tier 3, camera orientation} & Pan $90^\circ$ ($\numChunks{=}4$) & Pan $180^\circ$ ($\numChunks{=}8$) & Pan $360^\circ$ ($\numChunks{=}8$) \\
\cmidrule(l){2-4}
Sliding window         & $0.838$ / $0.561$ & $0.560$ / $\mathbf{0.634}$ & $0.406$ / $0.740$ \\
MemRoPE                & $0.844$ / $0.547$ & $0.561$ / $0.639$ & $0.394$ / $0.742$ \\
\textbf{\worldtrace}   & $\mathbf{0.911}$ / $\mathbf{0.453}$ & $\mathbf{0.784}$ / $0.649$ & $\mathbf{0.423}$ / $\mathbf{0.722}$ \\
\midrule
\textit{Tier 4, multi-revisit} & ABABA ($\numChunks{=}32$) & ABCBA ($\numChunks{=}20$) & ABCDBA ($\numChunks{=}20$) \\
\cmidrule(l){2-4}
Sliding window         & $0.462$ / $0.692$ & $0.571$ / $0.659$ & $0.595$ / $0.653$ \\
MemRoPE                & $0.496$ / $0.675$ & $0.584$ / $0.642$ & $0.602$ / $0.634$ \\
\textbf{\worldtrace}   & $\mathbf{0.791}$ / $\mathbf{0.518}$ & $\mathbf{0.807}$ / $\mathbf{0.515}$ & $\mathbf{0.808}$ / $\mathbf{0.503}$ \\
\bottomrule
\end{tabular}%
}
\end{table}

\subsection{Main Results}
\label{sec:loop}

We report DINOv2 similarity and LPIPS on \loopbench. The qualitative examples in \Fig{fig:landmark_qualitative} show the sliding-window baseline failing to reconstruct scene~A at revisits, whereas \worldtrace restores its appearance across the three geometries.
The \worldtrace row in \Tab{tab:loopbench} reports the combined variant, whose values also appear in \Tab{tab:ablation}. The four evaluation task families are \emph{topology} (ABA, ABCA, ABCDA), \emph{rollout length} (ABA at $\numChunks{\in}\{8,16,32\}$), \emph{camera-orientation} (pan $90^\circ$/$180^\circ$/$360^\circ$ with the agent staying at~A), and \emph{multi-revisit} (ABABA, ABCBA, ABCDBA). The WorldTrace variants use $\numSummarySlots{=}4$, $\numRecentSlots{=}2$.

\worldtrace has the highest mean DINOv2 score among the methods in \Tab{tab:loopbench} in every listed MG2 condition, improving over the sliding window by approximately $32.6\%$ when DINOv2 is averaged across the 12 conditions. Its LPIPS is also lower than the sliding window in all but Pan~$180^\circ$ ($0.649$ versus $0.634$). MemRoPE~\citep{kim2026memrope}, evaluated under the same attended-cache size, stays within $0.046$ DINOv2 of the sliding window across the listed conditions, with lower LPIPS in all but Pan~$180^\circ$ and Pan~$360^\circ$. The gap to the sliding window is larger at long horizons. In the horizon sweep (\Tab{tab:recall_horizon}), at $\numChunks{=}128$, \wtunified reaches $0.698$ DINOv2 / $0.572$ LPIPS, compared with $0.251$ / $0.751$ for the sliding window. We report cross-architecture results in \App{app:extra_results}.

\begin{table}[t]
\centering
\caption{\textbf{\worldtrace performance on \loopbench.} DINOv2 ($\uparrow$) / LPIPS ($\downarrow$) between the frames generated at the revisited location (\eg A in ABA) on the first visit and on the return on \loopbench, mean over $n{=}51$ rollouts, six-frame attended window. \wtfield keeps compressed summary slots only, \wtlandmark fills all $\numSummarySlots$ slots with verbatim frames, and \wtunified combines both, keeping a bank of verbatim frames alongside compressed summaries and exposing the one the queries select.}
\label{tab:ablation}
\small
\setlength{\tabcolsep}{4pt}
\resizebox{\columnwidth}{!}{%
\begin{tabular}{l ccc}
\toprule
\textit{Tier 1, topology} & ABA ($\loopK{=}1$, $\numChunks{=}16$) & ABCA ($\loopK{=}2$, $\numChunks{=}17$) & ABCDA ($\loopK{=}3$, $\numChunks{=}16$) \\
\cmidrule(l){2-4}
\wtfield               & $0.591$ / $0.636$ & $0.610$ / $0.640$ & $0.588$ / $0.637$ \\
\wtlandmark            & $0.723$ / $0.543$ & $0.709$ / $0.561$ & $0.721$ / $0.557$ \\
\textbf{\wtunified}    & $\mathbf{0.798}$ / $\mathbf{0.512}$ & $\mathbf{0.767}$ / $\mathbf{0.548}$ & $\mathbf{0.805}$ / $\mathbf{0.489}$ \\
\midrule
\textit{Tier 2, rollout length} & ABA short ($\numChunks{=}8$) & ABA ($\numChunks{=}16$) & ABA long ($\numChunks{=}32$) \\
\cmidrule(l){2-4}
\wtfield               & $0.785$ / $0.530$ & $0.591$ / $0.636$ & $0.409$ / $0.703$ \\
\wtlandmark            & $0.865$ / $0.466$ & $0.723$ / $0.543$ & $0.632$ / $0.588$ \\
\textbf{\wtunified}    & $\mathbf{0.891}$ / $\mathbf{0.460}$ & $\mathbf{0.798}$ / $\mathbf{0.512}$ & $\mathbf{0.750}$ / $\mathbf{0.537}$ \\
\midrule
\textit{Tier 3, camera orientation} & Pan $90^\circ$ ($\numChunks{=}4$) & Pan $180^\circ$ ($\numChunks{=}8$) & Pan $360^\circ$ ($\numChunks{=}8$) \\
\cmidrule(l){2-4}
\wtfield               & $0.818$ / $0.576$ & $0.555$ / $0.650$ & $0.378$ / $0.746$ \\
\wtlandmark            & $\mathbf{0.917}$ / $0.495$ & $0.731$ / $\mathbf{0.608}$ & $0.400$ / $0.729$ \\
\textbf{\wtunified}    & $0.911$ / $\mathbf{0.453}$ & $\mathbf{0.784}$ / $0.649$ & $\mathbf{0.423}$ / $\mathbf{0.722}$ \\
\midrule
\textit{Tier 4, multi-revisit} & ABABA ($\numChunks{=}32$) & ABCBA ($\numChunks{=}20$) & ABCDBA ($\numChunks{=}20$) \\
\cmidrule(l){2-4}
\wtfield               & $0.438$ / $0.696$ & $0.573$ / $0.657$ & $0.578$ / $0.649$ \\
\wtlandmark            & $0.719$ / $0.555$ & $0.709$ / $0.551$ & $0.735$ / $0.547$ \\
\textbf{\wtunified}    & $\mathbf{0.791}$ / $\mathbf{0.518}$ & $\mathbf{0.807}$ / $\mathbf{0.515}$ & $\mathbf{0.808}$ / $\mathbf{0.503}$ \\
\bottomrule
\end{tabular}%
}
\end{table}

\subsection{Ablations}
\label{sec:ablations}
We ablate the two writers by comparing a field-only memory (\wtfield), which keeps only averaged summaries, a landmark-only memory (\wtlandmark), which keeps only verbatim landmarks, and our unified \wtunified, which combines both under the attention-guided update of \Sec{subsec:combined}. All three share the same virtual positions, canonical keys, and six-frame attended window, so the comparison isolates what the slots store. Averaging alone does not restore recall. \wtfield stays within $0.03$ DINOv2 of the sliding window in every condition of \Tab{tab:ablation} (mean $0.576$ versus $0.587$), since averaging keeps distant history addressable but blurs the details that distinguish one place from another.

Verbatim landmarks drive recall. \wtlandmark raises mean DINOv2 to $0.715$ and improves over the sliding window in every condition except Pan~$360^\circ$ ($0.400$ versus $0.406$). Combining both writers is best. \wtunified reaches a mean DINOv2 of $0.778$ and is highest in 11 of 12 conditions, with the largest gains over \wtlandmark on long and multi-revisit rollouts such as ABA long ($0.750$ versus $0.632$) and ABCBA ($0.807$ versus $0.709$), suggesting that merging redundant observations frees slots for distinctive ones. The corresponding standard errors are reported in \Tab{tab:loopbench_full}. 

\section{Conclusion}
\label{sec:conclusion}

Long-horizon failure in autoregressive video world models is not only a matter of retaining more past content, but of keeping that content addressable. We show that temporal RoPE extrapolation can leave cached content unaddressable once rollouts move beyond the training horizon, and that averaging keys in the rotated space cancels what they encode. We introduce \worldtrace, a training-free memory framework that stores keys canonically and places every compressed slot at a distinct, in-distribution virtual position. We view memory compression as structured sparse attention, in which each slot is a structured projection of the frame history. Within this bounded memory bank, the model's own attention decides which observations to keep verbatim and which redundant ones to merge into summaries. On \loopbench, \worldtrace improves average DINOv2 recall similarity by $32.6\%$ over a sliding window at the same cache budget, without retraining. More broadly, these results suggest that memory for video world models should be designed around two questions, how to keep past content addressable and which visual traces to store for future queries. \worldtrace takes a step toward persistent interactive worlds that users can explore and revisit without losing the visual identity of previously seen places.

\section*{Acknowledgements}
We thank our colleagues at NVIDIA's Spatial Intelligence Lab and at Princeton University for many helpful discussions and valuable feedback.

\renewcommand{\bibfont}{\small\linespread{0.94}\selectfont}
\setlength{\bibsep}{2pt}
\bibliographystyle{plainnat}
\bibliography{references}

\appendix
\newpage
\clearpage

\newpage
\appendix
\appendixpage

\hypertarget{toc}{}
\startcontents[sections]
\begingroup
\small
\setlength{\parskip}{0pt}
\printcontents[sections]{l}{1}{\setcounter{tocdepth}{2}}
\endgroup

\pagestyle{fancy}
\renewcommand{\headrulewidth}{0pt}
\fancyhead{}
\fancyfoot[L]{\hyperlink{toc}{Back to Table of Contents}}
\fancyfoot[C]{\thepage}
\fancyfoot[R]{\hyperlink{abstract}{Back to the First Page}}

\newpage
\section{Notation}
\label{app:notation}

\begin{table}[!h]
\centering
\captionsetup{justification=centering}
\caption{\textbf{Glossary and notation.} }
\label{tab:notation_core}
\renewcommand{\arraystretch}{1.2}
\setlength{\aboverulesep}{0.2ex}
\setlength{\belowrulesep}{0.35ex}
\begin{threeparttable}
\begin{tabular}{l p{0.78\linewidth}}
\toprule
\textbf{Symbol} & \textbf{Description} \\
\midrule
\multicolumn{2}{l}{\textit{Cache Structure}} \\
$\numChunks$         & Total number of AR chunks in a rollout (generation length, $\framesPerBlock$ latent frames per chunk) \\
$\numSummarySlots$   & Number of summary slots in the KV cache (one latent-frame-sized KV payload each) \\
$\numRecentSlots$    & Number of recent-window slots (one verbatim latent frame each) \\
$\localAttnSize$     & Local attention / inference KV window in latent frames ($\numSummarySlots+\numRecentSlots = \localAttnSize$) \\
$\trainLen$          & Training-time KV-cache extent in AR chunks during Self-Forcing rollouts (distinct from the local attention window $\localAttnSize$) \\
$\framesPerBlock$    & Latent frames per AR chunk \\
$\numSourceFrames$   & Source frames compressed into one summary slot under \wtfield \\
$\numPastFrames$ & Number of past latent frames in the full-history reference, rather than the number retained by the bounded cache \\
$\slotIdx$           & Summary slot index, $\slotIdx = 0,\ldots,\numSummarySlots{-}1$ (0 = oldest) \\
$\recentCache$       & Recent-window cache containing verbatim KVs of the newest $\numRecentSlots$ latent frames \\
$\summaryCache$      & Summary cache containing $\numSummarySlots$ compressed slots indexed by $\slotIdx$ \\
\midrule
\multicolumn{2}{l}{\textit{Indices}} \\
$\idxFreq$         & RoPE frequency-pair index \\
$\idxSrc$          & Source-frame iteration index (inside $\sum\nolimits_{\idxSrc=1}^{\numSourceFrames}$) \\
\midrule
\multicolumn{2}{l}{\textit{Virtual Position Assignment (Slot Indexing)}} \\
$\qpos$    & Absolute timestamp of the current query frame. In the attention-approximation analysis, $q$ instead denotes a query vector \\
$\kpos$    & Absolute timestamp of a cached key frame \\
$\tminv$   & In-distribution lower bound $\max(0,\; \qpos - \trainOffset)$ \\
$\tmaxv$   & In-distribution upper bound for summary slots $\qpos - \numRecentSlots$ \\
$\tv$      & Virtual position \\
$\tvslot$  & Virtual position of slot $\slotIdx$, given by $\qpos - (\localAttnSize{-}1{-}\slotIdx)$, for the populated full-cache regime \\
\midrule
\multicolumn{2}{l}{\loopbench \textit{Benchmark}} \\
$\loopK$           & Number of intermediate waypoints between departure and return to scene~A \\
$\loopR$           & Pattern-specific revisit count, measuring returns to A in ABABA and visits to B in ABCBA and ABCDBA (all shown with $\loopR{=}2$) \\
\bottomrule
\end{tabular}
\end{threeparttable}
\end{table}

\clearpage
\begin{table}[!h]
\ContinuedFloat
\centering
\captionsetup{justification=centering}
\caption[]{\textbf{Glossary and notation (continued).}}
\renewcommand{\arraystretch}{1.2}
\setlength{\aboverulesep}{0.2ex}
\setlength{\belowrulesep}{0.35ex}
\begin{threeparttable}
\begin{tabular}{l p{0.78\linewidth}}
\toprule
\textbf{Symbol} & \textbf{Description} \\
\midrule
\multicolumn{2}{l}{\textit{RoPE Parameters}} \\
$\ropeBase$         & RoPE base parameter ($\ropeBase{=}10000$ for MG2) \\
$\ropefreq$         & RoPE angular frequency for temporal head-dimension pair $\idxFreq$ ($\ropefreq = \ropeBase^{-\idxFreq/\numTempPairs}$) \\
$\numTempPairs$, $\numSpatialPairsH$, $\numSpatialPairsW$ & Number of temporal, height, and width RoPE complex pairs per head \\
$\Rot{\alpha}$      & Rotation matrix by angle $\alpha$ (applied per frequency pair) \\
\midrule
\multicolumn{2}{l}{\textit{Key / Query Variants}} \\
$\Krot$             & RoPE-rotated key at absolute position $\timeSym_{\idxSrc}$, frequency pair $\idxFreq$ \\
$\Kcx$              & Canonical (unrotated) key content $\Rot{-\ropefreq \timeSym_{\idxSrc}}\,\Krot$ \\
$\Kcxmean$          & Canonical mean across $\numSourceFrames$ source frames  \\
$\Knaive$           & Naive RoPE-space average of $\numSourceFrames$ rotated keys \\
$\Kfield$           & \wtfield compressed key at virtual position $\tv$  \\
$\Klandslot$        & \wtlandmark key re-encoded at virtual position $\tv$ from frozen canonical content  \\
$\landmarkTime$     & Original timestamp of a selected landmark frame \\
$\Klandsrc$         & RoPE-rotated landmark source key at $\landmarkTime$, frequency pair $\idxFreq$ \\
$\sbThreshold$      & Scene-boundary threshold on the cosine distance between canonical keys of consecutive frames \\
$\Qrope$            & RoPE-rotated query at $\qpos$, frequency pair $\idxFreq$ \\
\midrule
\multicolumn{2}{l}{\textit{Attention Quantities}} \\
$\relOffset$        & Query-key temporal offset $\qpos - \kpos$ \\
$\trainOffset$      & Largest trained temporal offset ($\trainOffset{=}5$ for MG2). An in-distribution cache requires $\localAttnSize{-}1\leq\trainOffset$ \\
$\attnLogit$        & Attention-score contribution from frequency pair $\idxFreq$ at offset $\relOffset$ \\
$\contentTerm$      & Query-key content term in canonical coordinates for $\attnLogit$ \\
$\fullAttnWeights$ & Full-reference attention weights over $\numPastFrames$ frames for query vector $q$. The compression analysis uses a repositioned reference \\
$\querySet$    & Query set in the aggregate mismatch objective (\App{app:sparse_attention}) \\
$\mathbf{K}_{\mathrm{hist}}$ & Canonical source keys of the $\numPastFrames$ past frames at a fixed layer, head, and spatial token \\
$P$, $P_i$      & Projection matrix $P\in\mathbb{R}^{\localAttnSize\times\numPastFrames}$ and its $i$-th row, the weights with which cache slot $i$ combines past frames \\
$\Delta_{\mathrm{attn}}$ & Attention Mismatch between uncompressed and compressed attention (\App{app:metrics}) \\
\bottomrule
\end{tabular}
\end{threeparttable}
\end{table}

\section{Structured Attention Approximations through Selection and Averaging}
\label{app:sparse_attention}

With full attention, each new chunk attends to all $\numPastFrames$ past frames, so attention cost and the KV cache grow as $\bigO{\numPastFrames}$, and video world models typically cap this with a sliding window over the most recent frames. \worldtrace instead compresses the full history to $\localAttnSize \ll \numPastFrames$ slots through a projection matrix $\projMat \in \Reals^{\localAttnSize \times \numPastFrames}$, whose rows define which frames contribute to each slot. Throughout this section, $\keySym$ denotes the canonical (unrotated) key matrix on which $\projMat$ acts, and the slot-rank rotation of \Def{def:vpos} is applied afterwards at attention time.

\subsection{Full and Projected Attention}
\label{app:proj_setup}
For a single query $\queryVec$, the standard attention output is $\attnOp(\queryVec,\keySym,\valueSym) = \softmaxOp(\queryVec\,\keySym^\top)\,\valueSym = \fullAttnWeights\,\valueSym$, a weighted sum over all $\numPastFrames$ value vectors with $\fullAttnWeights \in \simplex{\numPastFrames}$, where $\simplex{\numPastFrames}$ is the probability simplex over the $\numPastFrames$ past frames (non-negative weights summing to one).
The \emph{projected} attention approximation replaces the $\numPastFrames$-token sequence with $\localAttnSize \ll \numPastFrames$ prototype tokens (one per compressed slot) defined by a row-stochastic matrix $\projMat \in \Reals^{\localAttnSize \times \numPastFrames}$, meaning each row of $\projMat$ is itself a distribution over past frames. The projected attention is:
\begin{equation}
  \projAttnOp(\queryVec,\keySym,\valueSym;\,\projMat) \;=\; \softmaxOp(\queryVec\,(\projMat\keySym)^\top)\,\projMat\valueSym,
  \label{eq:hat_attn}
\end{equation}
where $\protoAttnWeights = \softmaxOp(\queryVec\,(\projMat\keySym)^\top) \in \simplex{\localAttnSize}$ are the prototype attention weights, now a distribution over the $\localAttnSize$ compressed slots rather than the $\numPastFrames$ frames, and $\projMat_i \valueSym$ is the compressed value for slot $i$.

\paragraph{Relation to prior structured-attention approximations.}
\Eq{eq:hat_attn} is the low-rank factorization of attention studied by efficient transformers, whose projections are learned end-to-end~\citep{wang2020linformer,xiong2021nystromformer,vyas2020clustered,bolya2023tome,vasylenko2025sparse,leng2025hsa}. \worldtrace instead applies the projection \emph{training-free} to a pretrained model's KV cache. Because of RoPE, $\projMat$ acts in the canonical key domain to avoid phase cancellation (\Sec{sec:phase_cancel_failure}), and the compressed slots are re-encoded at slot-rank virtual positions (\Def{def:vpos}) to stay addressable. The row structure of $\projMat$ (uniform averaging for \wtfield, one-hot selection for \wtlandmark) determines whether compression better approximates group-level or frame-level attention, as formalized below.

\subsection{Structured Attention Approximations Unified by $\projMat$}
\label{app:p_variants}

The four cache structures (\Fig{fig:proj_matrices}) in our experiments correspond to four structured choices of~$\projMat$.

\paragraph{Sliding window attention.}
Here $\projMat = [\,0\;\;\identMat{\localAttnSize}\,]$ makes each row a one-hot vector selecting one of the most recent $\localAttnSize$ frames, discarding all earlier history (\Fig{fig:proj_matrices}a). Adding an \emph{attention sink anchor} is a simple modification (\Fig{fig:proj_matrices}b) in which row 0 is pinned to the first (sink) frame while the remaining $\localAttnSize{-}1$ rows select the most recent $\localAttnSize{-}1$ frames verbatim, always keeping a global context token in cache regardless of generation length. This anchor mitigates the collapse of the attention distribution onto only local tokens at long generation horizons~\citep{xiao2024streamingllm}.

\paragraph{\worldtrace writers as grouped and sparse attention.}
In both writers, rows for the $\numRecentSlots$ recent slots are identity rows as above. For \wtfield, each remaining row is a uniform average over the $\numSourceFrames$ frames assigned to its summary slot (\Fig{fig:proj_matrices}c), with $\projMat_{i,t} = 1/\numSourceFrames$ if frame $t$ is assigned to slot $i$ and $0$ otherwise. For \wtlandmark, summary rows are one-hot selectors at the detected scene-entry frames (\Fig{fig:proj_matrices}d), so the row storing the landmark at timestamp $\landmarkTime$ has $\projMat_{i,t} = 1$ for $t = \landmarkTime$ and $0$ otherwise, copying the frozen key $\projMat_i \keySym = \keySym[\landmarkTime]$ of \Eq{eq:landmark}. The effective attention weights $\protoAttnWeights\projMat$ are therefore supported only on retained frames for \wtlandmark, giving structured sparse attention, and spread with equal weights within each temporal group for \wtfield, giving grouped attention.

\newcommand{\pmatapp}[7]{%
  \begin{tikzpicture}[x=#3,y=#3,baseline=(current bounding box.center)]
    \draw[gray!30, line width=0.3pt]
      (0,0) grid [xstep=1,ystep=1] (#2,#1);
    \foreach \r/\c in {#5}{%
      \fill[#4!60] (\c,{#1-1-\r}) rectangle ++(1,1);}
    \draw[black, line width=0.8pt] (0,0) rectangle (#2,#1);
    \ifnum#6>0
      \node[font=\scriptsize,gray,rotate=90,anchor=center,overlay] at (-0.6,{#1/2}) {$\localAttnSize$ slots};
    \fi
    \node[font=\scriptsize,gray,below=3pt] at (#2/2,0) {$\numPastFrames$ frames};
    \ifnum#7>0
      \draw[decorate,decoration={brace,amplitude=5pt},gray,line width=0.6pt]
        ({#2+0.15},{#1}) -- ({#2+0.15},{#1-#7});
      \node[font=\scriptsize,gray,anchor=west] at ({#2+0.5},{#1-#7*0.5}) {$N_s$};
      \draw[decorate,decoration={brace,amplitude=5pt},gray,line width=0.6pt]
        ({#2+0.15},{#1-#7}) -- ({#2+0.15},0);
      \node[font=\scriptsize,gray,anchor=west] at ({#2+0.5},{(#1-#7)*0.5}) {$N_r$};
    \fi
  \end{tikzpicture}}

\begin{figure}[t!]
\centering
\resizebox{\linewidth}{!}{%
\begin{tabular}{@{}cccc@{}}
  \pmatapp{6}{10}{0.38cm}{teal}{0/4, 1/5, 2/6, 3/7, 4/8, 5/9}{4}{0}
  &
  \pmatapp{6}{10}{0.38cm}{teal}{0/0, 1/5, 2/6, 3/7, 4/8, 5/9}{0}{0}
  &
  \pmatapp{6}{10}{0.38cm}{blue}{0/0, 0/1, 1/2, 1/3, 2/4, 2/5, 3/6, 3/7, 4/8, 5/9}{0}{0}
  &
  \pmatapp{6}{10}{0.38cm}{orange}{0/1, 1/3, 2/4, 3/6, 4/8, 5/9}{0}{4}
  \\[5pt]
  \small\textbf{(a)} Sliding window & \small\textbf{(b)} + attention sink & \small\textbf{(c)} \wtfield & \small\textbf{(d)} \wtlandmark
\end{tabular}}
\caption{\textbf{Projection matrices $\projMat$ for the four cache structures}.
\textbf{(a)}~\emph{Sliding window} retains the last $\localAttnSize$ frames verbatim.
\textbf{(b)}~\emph{Sliding window with attention sink} retains the initial frame as an attention sink and fills the remaining slots with verbatim recent frames.
\textbf{(c)}~\emph{\wtfield} keeps the newest frames verbatim in recent rows and averages a group of earlier frames in each summary slot.
\textbf{(d)}~\emph{\wtlandmark} selects scene-entry frames with one-hot summary rows and keeps recent rows verbatim.}
\label{fig:proj_matrices}
\end{figure}

\subsection{Approximation Quality and Error Decomposition}
\label{app:proj_quality}

The cache contents, and hence $\projMat$, change as generation proceeds, since the recent window slides and summary rows are filled or overwritten. All statements in this and the following subsections therefore refer to one generation step at a time, where $\numPastFrames$ and $\projMat$ denote the history length and the projection realized by the cache at the current step, and $\querySet$ denotes the queries of the chunk being generated (one attention row per query token and head). Every bound holds at each step separately, and no single static $\projMat$ is claimed to serve all steps.

\paragraph{Two sources of error.}
The compressed cache differs from an unbounded cache in two independent ways. Keys are \emph{repositioned} (re-encoded at slot-rank virtual positions rather than their original timestamps), and keys are \emph{compressed} (merged through $\projMat$). Writing $\attnOp_{\mathrm{orig}}$ for full attention with keys at their original positions, $\attnOp$ for full attention with keys re-encoded at virtual positions, and $\projAttnOp$ for the projected attention of \Eq{eq:hat_attn} on the re-encoded cache, the total error decomposes as:
\[
\|\projAttnOp - \attnOp_{\mathrm{orig}}\|
\;\le\;
\underbrace{\|\attnOp - \attnOp_{\mathrm{orig}}\|}_{\text{repositioning}}
\;+\;
\underbrace{\|\projAttnOp - \attnOp\|}_{\text{compression}} .
\]
The repositioning term is not an artifact of compression but the trade-off for addressability. Beyond the training horizon the original offsets are out-of-distribution and attention degrades outright (\Sec{sec:rope_ood}), so slot-rank re-encoding trades an uncontrolled OOD failure for a bounded, in-distribution position shift. We do not bound this term analytically. The compression bound below does not quantify the effect of repositioning. \Prop{prop:approx_bound} below bounds the compression term. All statements in this section are relative to the repositioned reference $\attnOp$.

For a single query $\queryVec$, let $\protoAttnWeights = \softmaxOp(\queryVec\,(\projMat\keySym)^\top) \in \simplex{\localAttnSize}$ be the prototype attention weights under the compressed cache, and define the output error $\resid = \projAttnOp(\queryVec,\keySym,\valueSym;\,\projMat) - \attnOp(\queryVec,\keySym,\valueSym)$.

\begin{proposition}[Approximation bound]
\label{prop:approx_bound}
\[
  \|\resid\|_2 \;\leq\; \underbrace{\|\protoAttnWeights \projMat - \fullAttnWeights\|_1}_{\text{distribution mismatch}} \;\cdot\; \|\valueSym\|_{\infty,\mathrm{row}},
\]
where $\|\valueSym\|_{\infty,\mathrm{row}} = \max_t \|\valueSym[t]\|_2$.
\end{proposition}
\begin{proof}
By associativity $\projAttnOp = \protoAttnWeights(\projMat\valueSym) = (\protoAttnWeights\projMat)\valueSym$, so $\resid = (\protoAttnWeights\projMat - \fullAttnWeights)\valueSym = \sum_t c_t\,\valueSym[t]$ with $c = \protoAttnWeights\projMat - \fullAttnWeights$. Applying the vector-valued H\"{o}lder inequality ($\ell_1$/$\ell_\infty$ duality) to the rows of $\valueSym$, we obtain $\|\resid\|_2 \leq \|c\|_1\,\|\valueSym\|_{\infty,\mathrm{row}}$.
\end{proof}

The bound motivates minimizing distribution mismatch $\|\protoAttnWeights \projMat - \fullAttnWeights\|_1$ independently of value geometry. For \wtfield, linearity of the dot product ensures the compressed key $\Kfield$ reproduces the mean attention score of its source keys at the shared virtual position, limiting per-slot mismatch when query preferences vary smoothly across a group. For \wtlandmark, a verbatim scene-entry key keeps the mismatch small for queries concentrated on that event.

\subsection{Reduction to Nonnegative Matrix Factorization}
\label{app:proj_nmf}

Nonnegative matrix factorization (NMF) approximates a nonnegative matrix as the product of two low-rank nonnegative factors~\citep{lee1999learning}, which form a compact summary of the original data. A relaxation of attention compression through $\projMat$ has this form. Treat $\protoAttnWeights \in \simplex{\localAttnSize}$ as a free variable (decoupling it from the softmax in \Eq{eq:hat_attn}), and define $\stackedFull \in \Reals^{|\querySet|\times \numPastFrames}$ as the matrix whose rows are the full attention distributions $\fullAttnWeights$, one row per query in $\querySet$, and $\stackedProto \in \Reals^{|\querySet|\times \localAttnSize}$ as the matrix whose rows are the corresponding prototype weights $\protoAttnWeights$. The aggregate distribution-mismatch objective then becomes:
\[
  \min_{\stackedProto\geq 0,\;\projMat\geq 0} \|\stackedProto\projMat - \stackedFull\|_1 \quad\text{subject to row-sum constraints on }\stackedProto\text{ and }\projMat,
\]
which is NMF with an $\ell_1$ fidelity term, in the spirit of robust NMF formulations~\citep{kong2011robust}. This relaxation serves the analysis only, because at inference $\protoAttnWeights$ is not free but determined by $\projMat$ through the softmax, $\protoAttnWeights = \softmaxOp(\queryVec\,(\projMat\keySym)^\top)$ for each query $\queryVec$ and key matrix $\keySym$.
Under this factorization, \wtfield rows are group means, like $k$-means centroids, and \wtlandmark rows are retained exemplars, like $k$-medoids. The unified \worldtrace memory mixes both row types in one $\projMat$, using attention mismatch to choose which history to average and which to keep verbatim.

\section{Additional Properties}
\label{app:wtfield_details}

\paragraph{Temporal RoPE at long horizons.}
\label{app:rope_background}
We provide temporal-RoPE background for the addressability failure in \Sec{sec:rope_ood}. MG2-1.3B inherits the 3D RoPE split of Wan~2.1~\citep{wan2025}, with $2\numTempPairs{=}44$ temporal and $2\numSpatialPairsH{=}2\numSpatialPairsW{=}42$ height and width dimensions per head and base $\ropeBase = 10000$, so the temporal frequencies $\ropefreq = \ropeBase^{-\idxFreq/\numTempPairs}$ range from $1.0$ to about $1.5 \times 10^{-4}$. Frequency $\idxFreq$ enters through the rotation $\Rot{\ropefreq\relOffset}$, and following YaRN's NTK-by-parts view~\citep{peng2023yarn}, we compare each component's wavelength to the trained context. At $\trainOffset = 5$, components $\idxFreq \geq 10$ rotate by less than $0.08$\,rad, while at $|\relOffset| = 30$ the three fastest exceed $3\pi$\,rad (\Fig{fig:ood_heatmap}), and an unbounded rollout pushes ever more components beyond $\pi$. This describes unwrapped phase, not measured attention quality, and low-frequency channels can also fail under extrapolation~\citep{barbero2025round}. \worldtrace keeps all temporal offsets within the trained window.

\begin{figure}[t!]
\centering
\includegraphics[width=\linewidth]{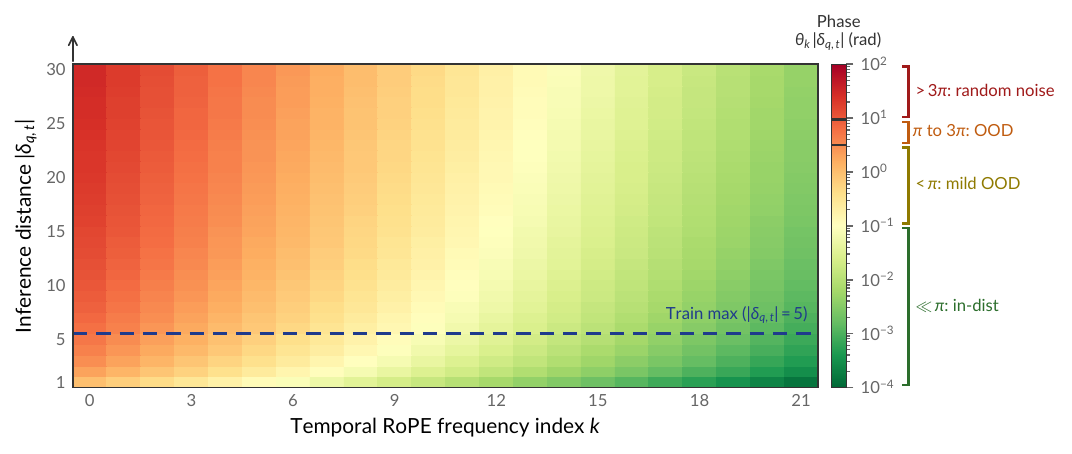}
\caption{\textbf{Temporal RoPE rotation angle per frequency.} Per-frequency unwrapped phase $\theta_k |\delta_{q,t}|$ on MG2-1.3B. In this figure, $k$ indexes temporal frequencies and $\delta_{q,t}=q-t$ is the query-key offset. Color shows phase magnitude on a log scale, and the dashed line marks the training maximum $|\delta_{q,t}|{=}5$. The bracket labels are schematic phase bands, not empirically established thresholds for attention or recall failure.}
\label{fig:ood_heatmap}
\end{figure}

\section{Related Work}
\label{app:related}

\subsection{Memory-Augmented Transformers}
\label{app:related_memory_aug}

Memory-augmented transformers pair a recent verbatim window with a compressed record of older context, realized as segment-level recurrence with relative positions~\citep{dai2019transformerxl}, a compressed memory of older tokens~\citep{rae2020compressive}, external $k$NN memory~\citep{wu2022memorizing}, learned memory tokens passed between segments~\citep{bulatov2022rmt,hutchins2022block}, and a fused sliding window plus compressive memory~\citep{munkhdalai2024infini}. \worldtrace adopts the same two-tier structure, with a recent verbatim window plus compressed summary slots.
Landmark Attention~\citep{mohtashami2023landmark} shares vocabulary with \wtlandmark, but its landmarks are \emph{trained} representatives that gate attention to off-cache blocks at 1D position, whereas \wtlandmark is training-free, stores verbatim canonical-frame keys inside the $O(1)$ cache, and freezes them against unrotate-and-rerotate drift in 3D RoPE. The broader spectrum toward compressed state spans linear attention~\citep{katharopoulos2020linear}, Mamba~\citep{gu2023mamba}, and Test-Time Training~\citep{sun2024ttt}. We target the fixed-budget, training-free setting in AR video world models.

\subsection{KV Cache Compression for Large Language Models}
\label{app:related_kv_llm}

The three main families of LLM KV-cache methods are \emph{window/eviction}~\citep{beltagy2020longformer,xiao2024streamingllm,zhang2023h2o,li2024snapkv,cai2025pyramidkv}, \emph{merging}~\citep{bolya2023tome,keepkv2026}, and \emph{quantization}~\citep{liu2024kivi}, the last orthogonal to the position-OOD failure we target. DuoAttention~\citep{xiao2025duo} splits heads into retrieval and streaming families, parallel to our recent-window and summary-slot split. These methods change \emph{content} only, picking which tokens to keep or merge rather than which virtual position a summary spanning several merged frames should occupy. The closest exceptions touch per-token positional structure~\citep{elitekv2025,a2ats2025} or unrotate and rerotate individually retained keys~\citep{corallo2024finch,kvpress2024}, but none reassign summary-slot positions. Our setting also differs from LLM compression in that keys accumulate across independently denoised AR video chunks, so importance needs to be ranked under denoising-timestep-conditioned queries. Under such queries, angular-concentration scores~\citep{mao2026triattention} lose the stable pre-RoPE key geometry they assume, and layer-budget reallocations~\citep{cai2025pyramidkv,zhang2024simlayerkv} tune which layers keep tokens rather than how summary slots are addressed.
\wtfield instead averages canonical keys and assigns slot-rank positions.

\subsection{Position Extrapolation for {RoPE}}
\label{app:related_rope}

RoPE~\citep{su2024rope} degrades past trained relative offsets. 
Existing fixes span index rescaling~\citep{chen2023posint,peng2023yarn,shang2025longrope2}, alternative distance biases~\citep{press2022alibi,sun2023xpos}, position-free training~\citep{kazemnejad2023nope}, and frequency- or content-aware encodings~\citep{golovneva2024cope,hua2025fope}, while vision and video extensions add 2D/3D splits~\citep{heo2024ropevit,wei2025videorope}, camera-relative encodings~\citep{li2026cameras}, and training-free diffusion extrapolation~\citep{zhao2025riflex}. All of this remaps positions for individually retained tokens or fixed-length inputs. Our cache instead needs to be \emph{compressed}, since a fixed $O(1)$ budget cannot keep every past frame verbatim. Each merged summary then needs an explicit virtual position, a quantity that does not arise when tokens keep separate identities, and \worldtrace assigns slot-rank offsets for that role without retraining.
\citet{wang2025anchor} tie position to numerical precision, relevant to the frozen-key discussion in \Sec{subsec:wtlandmark}. UCM~\citep{xu2026ucm} and MosaicMem~\citep{yu2026mosaicmem} warp 3D positions for camera-controlled worlds via trained geometric operators.

\subsection{Autoregressive Video World Models}
\label{app:related_video_wm}

Interactive AR video world models~\citep{genie3_2025,matrixgame2025,oasis2024,valevski2025gamengen} generate via chunk-wise diffusion prediction, dominantly under the Self-Forcing paradigm~\citep{selfforcing2025,chen2024diffusionforcing,song2025dfot,yin2025causvid}, which aligns training rollouts with the model's own KV cache and inherits a local attention window not sized for arbitrarily long inference. Training-time horizon extensions~\citep{cui2026selfforcingpp,huang2026vid2world} and memory-specific methods (\Sec{app:related_video_memory}) address this limit through finetuning. Foundation backbones~\citep{wan2025} and the world-model lineage from \citet{ha2018world} through Dreamer~\citep{hafner2025dreamerv3} to closed-loop driving simulators~\citep{basant2026nvidia} motivate interactive rollouts beyond passive generation.

\subsection{Memory and KV Cache for Video World Models}
\label{app:related_video_memory}

\paragraph{Inference-time KV cache.}
Closest to \worldtrace, inference-time methods modify cache read/write without retraining. \emph{Position-side} fixes remap or cap temporal offsets via block-relative RoPE with aggressive flushing~\citep{infinityrope2026}, inference-time position adjustment~\citep{cui2026lol,li2026flex}, or per-token virtual positions with cross-frame decay~\citep{li2026packcache}. \emph{Content-side} writers decide what to merge, prune, or retain, through re-aligned sink tokens~\citep{yi2025deepforcing}, dual-rate EMA summaries~\citep{kim2026memrope}, importance pruning~\citep{chen2026pafukv,li2026rollingsink}, and temporal-correspondence merging~\citep{samuel2026tempcache}. FlowCache~\citep{ma2026flowcache} targets per-step compute rather than long-horizon recall, but its recent-summary structure parallels ours. MemRoPE already jointly uses canonical memory updates and contiguous online indices~\citep{kim2026memrope}. \worldtrace studies canonical averaging and retained-frame memory for episodic recall under a fixed attended-cache size, evaluated with controlled scene revisits on \loopbench.

\paragraph{Training-time extensions.}
Training-time methods finetune the generator itself for longer effective context. One line pairs partitioned caches or memory hierarchies with RoPE re-alignment or long-context teacher supervision~\citep{chen2026groundedforcing,yang2026anchorforcing,chen2026contextforcing,li2026hybridforcing,li2025svi,zhao2026relaxforcing,liu2026rollforce}. Memory-decomposition approaches assign different parts of history to dedicated modules~\citep{stapf2026come,wu2025spmem,cai2026moc}.
Architectural long-horizon models train for extended rollouts rather than patching a fixed checkpoint~\citep{gu2025far,oshima2025worldpack,sun2025worldplay,yu2025videossm,zhu2026sana}, and Sparse Forcing~\citep{xu2026sparseforcing} learns block-sparse attention with persistent anchors end-to-end. \citet{po2025long} extend the effective context of a video world model with a trained state-space memory and evaluate it on Memory Maze with a spatial retrieval task, where the model is given a trajectory as context and then backtracks the exact action sequence to the starting position. \worldtrace needs no retraining, restoring addressability via cache read/write on a fixed checkpoint.

\paragraph{Geometry-conditioned memory.}
Trained methods decide what to keep in memory based on camera pose, actions, or scene geometry, drawing on
pose- and action-annotated compressed latents~\citep{hong2025relic}, warped positions or latents for camera-controlled worlds~\citep{yu2026mosaicmem,xu2026ucm,mao2026packforcing}, co-visibility-based selection~\citep{gao2026memcam}, pose-tagged frame re-injection~\citep{xiao2025worldmem}, and camera-aware memory gating~\citep{guo2026memorizewhenneeded}. \worldtrace is geometry-free, replacing pose-, action-, or co-visibility-driven memory selection with slot-rank positions and canonical cache writes.

\paragraph{External memory banks and retrieval.}
Rather than compressing inside the KV cache, these systems store or retrieve frames outside it and re-inject them via auxiliary attention or re-encoding, through FOV-based frame selection~\citep{matrixgame3_2026}, importance-weighted frame contexts~\citep{zhang2025framepack}, first-chunk appearance anchoring~\citep{henschel2025streamingt2v}, cached latent banks~\citep{wu2025corgi,li2025vmem,yu2025contextmemory}, and global-state retrieval~\citep{chen2025vrag}. The evaluated \wtunified variant also uses a bank outside the attended cache. It retains frame KVs and selects an entry using the model's attention queries (\Tab{tab:ablation}). The bank and query storage add memory beyond the fixed attended-cache size.

\paragraph{Comparison.}
In \Tab{tab:method_comparison}, we qualitatively compare memory approaches by considering whether positional OOD is fixed, whether memory is bounded, whether intermediate history is preserved, and whether retraining is required. For evaluation, WorldScore~\citep{duan2025worldscore}, MIND~\citep{ye2026mind}, and VBench-2.0~\citep{zheng2025vbench2} assess long-horizon quality, and \citet{lian2025memorybench} targets revisit consistency in Minecraft. The Memory Maze retrieval task of \citet{po2025long} is the closest published revisit test to ours and scores a backtracked trajectory against the recorded ground-truth frames, whereas \loopbench sweeps loop topology, rollout length, camera orientation, and multi-revisit patterns and scores each return against the model's own earlier generation at matched poses, so it needs no ground-truth footage of the revisited scene.

\begin{table}[t!]
\centering\small
\caption{\textbf{Memory approach comparison.} Representative baselines and inference-time cache methods on four axes. Memory complexity in this table refers to the attended KV cache.}
\label{tab:method_comparison}
\setlength{\tabcolsep}{4pt}
\resizebox{\columnwidth}{!}{%
\begin{tabular}{lcccc}
\toprule
Method & Fixes pos.\ OOD? & Bounded attended cache? & Preserves history? & Training-free? \\
\midrule
Full KV & No & No ($O(\numChunks)$) & Yes (exact) & Yes \\
Sliding window & N/A (no long-range reads) & Yes ($O(1)$) & Recent only & Yes \\
KV Flush $+$ Block-relative~\citep{infinityrope2026} & Yes & Yes ($O(1)$) & No  & Yes \\
Naive avg.\ $+$ Block-relative & Partial (slots collapse) & Yes ($O(1)$) & No (distorted)  & Yes \\
MemRoPE~\citep{kim2026memrope} & Yes  & Yes ($O(1)$) & EMA-diluted & Yes \\
Deep Forcing~\citep{yi2025deepforcing} & Partial & Yes ($O(1)$) & Sink plus pruned recent & Yes \\
\textbf{\worldtrace (ours)} & \textbf{Yes} & \textbf{Yes ($O(1)$)} & \textbf{Yes} & \textbf{Yes} \\
\bottomrule
\end{tabular}%
}
\end{table}

\section{Evaluation Metrics}
\label{app:metrics}

We evaluate memory through recall across revisits and attention fidelity under compression (\Tab{tab:metric_definitions}). Recall compares generated views from two visits to the same location, using the earlier view as the reference. Attention fidelity measures the effect of compression under matched queries and history.

\begin{table}[t!]
\centering
\caption{\textbf{Evaluation metrics.} Higher is better for metrics marked $\uparrow$, and lower is better for those marked $\downarrow$.}
\label{tab:metric_definitions}
\resizebox{\columnwidth}{!}{%
\begin{tabular}{llcp{6cm}}
\toprule
\textbf{Metric} & \textbf{Evaluation focus} & \textbf{Better} & \textbf{Definition} \\
\midrule
Scene Consistency (DINOv2) & Primary recall & $\uparrow$ &
DINOv2 feature cosine similarity across revisits. \\

Return LPIPS & Primary recall & $\downarrow$ &
Perceptual distance between the same revisited views. \\

\midrule
Attention Mismatch ($\Delta_{\mathrm{attn}}$) & Compression diagnostic & $\downarrow$ &
Attention-distribution difference induced by compression. \\

\bottomrule
\end{tabular}%
}
\end{table}

\paragraph{Recall.}
Scene Consistency is the cosine similarity between DINOv2~\citep{oquab2024dinov2} features of the frames generated at the revisited location (\eg A in ABA) on the first visit and on the return (higher is better). Return LPIPS~\citep{zhang2018perceptual} is the perceptual distance between the same two frames (lower is better). DINOv2 captures overall scene content and structure, while LPIPS is more sensitive to appearance and spatial layout. Neither alone establishes exact reconstruction, and viewpoint differences can affect both. We always use the frame that ends the return, without searching for a best-matching return frame. For DINOv2, we use the ViT-S/14 class-token embedding on frames resized to $252\times448$ with standard ImageNet normalization. For LPIPS, we use the VGG backbone on full-resolution frames.

\paragraph{Compression fidelity.}
Attention Mismatch compares uncompressed and compressed attention distributions under identical queries and source history, mapping compressed-slot weights back to the observations they represent. This controlled comparison measures how compression changes attention allocation, rather than comparing attention from independently diverging rollouts. Lower mismatch indicates greater attention fidelity but does not guarantee improved recall. We therefore use it as a compression diagnostic alongside the two recall metrics. Measurements are in \App{app:attn_mismatch}.

\section{Additional Experimental Results}
\label{app:extra_results}

\begin{table}[t]
\centering
\caption{\textbf{\loopbench results with standard errors.} DINOv2 ($\uparrow$) / LPIPS ($\downarrow$), mean$_{\pm\text{s.e.m.}}$ over $n{=}51$ rollouts. All methods share the six-frame attended window ($\numSummarySlots{=}4$, $\numRecentSlots{=}2$). }
\label{tab:loopbench_full}
\small
\setlength{\tabcolsep}{4pt}
\resizebox{\columnwidth}{!}{%
\begin{tabular}{l ccc}
\toprule
\textit{Tier 1, topology} & ABA ($\loopK{=}1$, $\numChunks{=}16$) & ABCA ($\loopK{=}2$, $\numChunks{=}17$) & ABCDA ($\loopK{=}3$, $\numChunks{=}16$) \\
\cmidrule(l){2-4}
Sliding window         & $0.577_{\pm0.026}$ / $0.643_{\pm0.008}$ & $0.627_{\pm0.018}$ / $0.643_{\pm0.008}$ & $0.599_{\pm0.024}$ / $0.640_{\pm0.008}$ \\
\wtfield               & $0.591_{\pm0.023}$ / $0.636_{\pm0.008}$ & $0.610_{\pm0.021}$ / $0.640_{\pm0.008}$ & $0.588_{\pm0.021}$ / $0.637_{\pm0.008}$ \\
\wtlandmark            & $0.723_{\pm0.024}$ / $0.543_{\pm0.011}$ & $0.709_{\pm0.024}$ / $0.561_{\pm0.011}$ & $0.721_{\pm0.022}$ / $0.557_{\pm0.011}$ \\
\textbf{\wtunified}    & $\mathbf{0.798}_{\pm0.020}$ / $\mathbf{0.512}_{\pm0.012}$ & $\mathbf{0.767}_{\pm0.023}$ / $\mathbf{0.548}_{\pm0.011}$ & $\mathbf{0.805}_{\pm0.018}$ / $\mathbf{0.489}_{\pm0.015}$ \\
\midrule
\textit{Tier 2, rollout length} & ABA short ($\numChunks{=}8$) & ABA ($\numChunks{=}16$) & ABA long ($\numChunks{=}32$) \\
\cmidrule(l){2-4}
Sliding window         & $0.805_{\pm0.011}$ / $0.525_{\pm0.008}$ & $0.577_{\pm0.026}$ / $0.643_{\pm0.008}$ & $0.422_{\pm0.025}$ / $0.700_{\pm0.005}$ \\
\wtfield               & $0.785_{\pm0.013}$ / $0.530_{\pm0.008}$ & $0.591_{\pm0.023}$ / $0.636_{\pm0.008}$ & $0.409_{\pm0.025}$ / $0.703_{\pm0.005}$ \\
\wtlandmark            & $0.865_{\pm0.010}$ / $0.466_{\pm0.010}$ & $0.723_{\pm0.024}$ / $0.543_{\pm0.011}$ & $0.632_{\pm0.025}$ / $0.588_{\pm0.011}$ \\
\textbf{\wtunified}    & $\mathbf{0.891}_{\pm0.007}$ / $\mathbf{0.460}_{\pm0.011}$ & $\mathbf{0.798}_{\pm0.020}$ / $\mathbf{0.512}_{\pm0.012}$ & $\mathbf{0.750}_{\pm0.024}$ / $\mathbf{0.537}_{\pm0.013}$ \\
\midrule
\textit{Tier 3, camera orientation} & Pan $90^\circ$ ($\numChunks{=}4$) & Pan $180^\circ$ ($\numChunks{=}8$) & Pan $360^\circ$ ($\numChunks{=}8$) \\
\cmidrule(l){2-4}
Sliding window         & $0.838_{\pm0.014}$ / $0.561_{\pm0.006}$ & $0.560_{\pm0.024}$ / $0.634_{\pm0.008}$ & $0.406_{\pm0.025}$ / $0.740_{\pm0.011}$ \\
\wtfield               & $0.818_{\pm0.015}$ / $0.576_{\pm0.007}$ & $0.555_{\pm0.023}$ / $0.650_{\pm0.008}$ & $0.378_{\pm0.024}$ / $0.746_{\pm0.012}$ \\
\wtlandmark            & $\mathbf{0.917}_{\pm0.004}$ / $0.495_{\pm0.005}$ & $0.731_{\pm0.021}$ / $\mathbf{0.608}_{\pm0.008}$ & $0.400_{\pm0.025}$ / $0.729_{\pm0.011}$ \\
\textbf{\wtunified}    & $0.911_{\pm0.008}$ / $\mathbf{0.453}_{\pm0.005}$ & $\mathbf{0.784}_{\pm0.013}$ / $0.649_{\pm0.008}$ & $\mathbf{0.423}_{\pm0.023}$ / $\mathbf{0.722}_{\pm0.011}$ \\
\midrule
\textit{Tier 4, multi-revisit} & ABABA ($\numChunks{=}32$) & ABCBA ($\numChunks{=}20$) & ABCDBA ($\numChunks{=}20$) \\
\cmidrule(l){2-4}
Sliding window         & $0.462_{\pm0.024}$ / $0.692_{\pm0.006}$ & $0.571_{\pm0.018}$ / $0.659_{\pm0.006}$ & $0.595_{\pm0.020}$ / $0.653_{\pm0.006}$ \\
\wtfield               & $0.438_{\pm0.025}$ / $0.696_{\pm0.006}$ & $0.573_{\pm0.021}$ / $0.657_{\pm0.007}$ & $0.578_{\pm0.020}$ / $0.649_{\pm0.007}$ \\
\wtlandmark            & $0.719_{\pm0.024}$ / $0.555_{\pm0.012}$ & $0.709_{\pm0.023}$ / $0.551_{\pm0.011}$ & $0.735_{\pm0.022}$ / $0.547_{\pm0.010}$ \\
\textbf{\wtunified}    & $\mathbf{0.791}_{\pm0.019}$ / $\mathbf{0.518}_{\pm0.012}$ & $\mathbf{0.807}_{\pm0.018}$ / $\mathbf{0.515}_{\pm0.011}$ & $\mathbf{0.808}_{\pm0.020}$ / $\mathbf{0.503}_{\pm0.011}$ \\
\bottomrule
\end{tabular}%
}
\end{table}

\subsection{Scene Consistency Episodic Recall Sweep}
\label{app:recall_sweep}

We report DINOv2 and LPIPS across $\numChunks \in \{16,32,48,64,128\}$ on ABA loops ($\numSummarySlots{=}4$, $\numRecentSlots{=}2$) in \Tab{tab:recall_horizon}. At $\numChunks{=}128$, the sliding window reaches $0.251$ DINOv2, while \wtunified reaches $0.698$. \wtunified also keeps lower LPIPS than the sliding window at every evaluated horizon.

\begin{table}[t]
\centering
\caption{\textbf{Episodic recall as the horizon grows.} ABA loops, DINOv2 ($\uparrow$) / LPIPS ($\downarrow$), $\numSummarySlots{=}4$, $\numRecentSlots{=}2$.
}
\label{tab:recall_horizon}
\small
\setlength{\tabcolsep}{4pt}
\resizebox{\columnwidth}{!}{%
\begin{tabular}{lccccc}
\toprule
Method & $\numChunks{=}16$ & $\numChunks{=}32$ & $\numChunks{=}48$ & $\numChunks{=}64$ & $\numChunks{=}128$ \\
\midrule
Sliding window               & 0.577 / 0.643 & 0.422 / 0.700 & 0.301 / 0.730 & 0.286 / 0.734 & 0.251 / 0.751 \\
\wtunified               & \textbf{0.798} / \textbf{0.512} & \textbf{0.750} / \textbf{0.537} & \textbf{0.714} / \textbf{0.560} & \textbf{0.745} / \textbf{0.534} & \textbf{0.698} / \textbf{0.572} \\
\bottomrule
\end{tabular}%
}
\end{table}

\subsection{Cross-Architecture Experiments}
\label{app:lingbot}

\Tab{tab:lingbot_recall} repeats \loopbench on LingBot-World~\citep{lingbot-world}. It is a 14B diffusion transformer with Pl\"{u}cker camera conditioning, whereas MG2 is a dense 1.3B model without it, so this comparison varies both scale and architecture. The camera loops are built on the release's example trajectory with the same waypoint structure as \Tab{tab:loopbench}. The cache holds $\numSummarySlots{=}2$ summary slots and $\numRecentSlots{=}16$ recent frames (18 cached frames plus the current three-frame chunk form the trained window of $21$ latent frames), every memory is applied per layer from that layer's own keys and values, and MemRoPE uses its released layout fitted to the two slots (one sink frame and the long EMA frame). Three trials per cell differ in the noise seed. \worldtrace leads on ABA, ABCDA, Pan $180^\circ$ and the three multi-revisit loops, by $0.2$ to $0.45$ DINOv2 over the sliding window on ABA, ABCBA and ABCDBA, which repeats the MG2 result on a backbone with explicit camera conditioning. MemRoPE leads on ABCA, ABA long and Pan $360^\circ$, and is marginally ahead in DINOv2 on ABA short. 

\begin{table}[t]
\centering
\caption{\textbf{\loopbench on LingBot-World~\citep{lingbot-world}.} DINOv2 ($\uparrow$) / LPIPS ($\downarrow$) between the frames generated at the revisited location (\eg A in ABA) on the first visit and on the return, mean over three trials, same four tiers as \Tab{tab:loopbench} with camera loops built on the release's example track. Every memory is applied per layer. MemRoPE is the released layout fitted to the two slots (one sink frame and the long EMA frame). }
\label{tab:lingbot_recall}
\small
\setlength{\tabcolsep}{4pt}
\resizebox{\columnwidth}{!}{%
\begin{tabular}{l ccc}
\toprule
\textit{Tier 1, topology} & ABA ($\loopK{=}1$, $\numChunks{=}16$) & ABCA ($\loopK{=}2$, $\numChunks{=}18$) & ABCDA ($\loopK{=}3$, $\numChunks{=}16$) \\
\cmidrule(l){2-4}
Sliding window         & $0.586$ / $0.670$ & $0.343$ / $0.694$ & $0.409$ / $0.693$ \\
MemRoPE                & $0.468$ / $0.666$ & $\mathbf{0.713}$ / $0.670$ & $0.545$ / $0.672$ \\
\textbf{\worldtrace}   & $\mathbf{0.812}$ / $\mathbf{0.656}$ & $0.449$ / $\mathbf{0.661}$ & $\mathbf{0.597}$ / $\mathbf{0.663}$ \\
\midrule
\textit{Tier 2, rollout length} & ABA short ($\numChunks{=}8$) & ABA ($\numChunks{=}16$) & ABA long ($\numChunks{=}32$) \\
\cmidrule(l){2-4}
Sliding window         & $0.864$ / $0.657$ & $0.586$ / $0.670$ & $0.319$ / $0.710$ \\
MemRoPE                & $\mathbf{0.873}$ / $0.656$ & $0.468$ / $0.666$ & $\mathbf{0.467}$ / $\mathbf{0.675}$ \\
\textbf{\worldtrace}   & $0.870$ / $\mathbf{0.652}$ & $\mathbf{0.812}$ / $\mathbf{0.656}$ & $0.274$ / $0.686$ \\
\midrule
\textit{Tier 3, camera orientation} & Pan $90^\circ$ ($\numChunks{=}4$) & Pan $180^\circ$ ($\numChunks{=}8$) & Pan $360^\circ$ ($\numChunks{=}8$) \\
\cmidrule(l){2-4}
Sliding window         & $0.628$ / $0.638$ & $0.408$ / $0.666$ & $0.457$ / $0.690$ \\
MemRoPE                & $0.628$ / $0.638$ & $0.392$ / $0.663$ & $\mathbf{0.522}$ / $\mathbf{0.670}$ \\
\textbf{\worldtrace}   & $0.628$ / $0.638$ & $\mathbf{0.540}$ / $\mathbf{0.653}$ & $0.460$ / $0.680$ \\
\midrule
\textit{Tier 4, multi-revisit} & ABABA ($\numChunks{=}32$) & ABCBA ($\numChunks{=}20$) & ABCDBA ($\numChunks{=}20$) \\
\cmidrule(l){2-4}
Sliding window         & $0.399$ / $0.701$ & $0.302$ / $0.693$ & $0.271$ / $0.706$ \\
MemRoPE                & $0.404$ / $0.696$ & $0.587$ / $\mathbf{0.674}$ & $0.682$ / $0.687$ \\
\textbf{\worldtrace}   & $\mathbf{0.430}$ / $\mathbf{0.675}$ & $\mathbf{0.748}$ / $0.680$ & $\mathbf{0.722}$ / $\mathbf{0.679}$ \\
\bottomrule
\end{tabular}%
}
\end{table}

\subsection{Attention Mismatch of Compressed Slots}
\label{app:attn_mismatch}

We measure the Attention Mismatch of \App{app:metrics} inside the rollout. At every attention call, the same queries are scored both against a summary slot and against the $\numSourceFrames$ frames it stands for, kept verbatim at the slot's position, so that only content differs. The slot's weight is mapped back onto its $\numSourceFrames$ frames, and $\Delta_{\mathrm{attn}}$ is half the $L_1$ distance between the two per-frame distributions, averaged over 96 queries per denoising pass, all heads, and all layers. We compare \worldtrace, whose averaged slot covers the newest temporal group, with a full-history average that puts all $n_{\mathrm{old}}$ old frames into a single slot. In \Tab{tab:attn_mismatch}, the mismatch grows with $\numSourceFrames$ and roughly doubles on the return leg, and the reference places far more attention on the summarized frames than the slot receives ($.41$ against $.04$ at $\numChunks{=}32$). This local comparison does not measure the loss of omitted frames or overall memory fidelity.

\begin{table}[t]
\centering
\caption{\textbf{Attention mismatch of compressed summary slots.} $\Delta_{\mathrm{attn}}$ ($\downarrow$), forward / return leg, on ABA loops for writers that replace $M$ verbatim frames with one slot ($\numSummarySlots{=}4$, $\numRecentSlots{=}2$), measured against attending to those $M$ frames verbatim at the slot's position. Each writer is measured on its newest summary slot, and full-history average puts all $n_{\mathrm{old}}$ old frames in one slot. Attention share (reference / slot) is at $\numChunks{=}32$, return leg, and DINOv2 ($\uparrow$) is for the same single-seed rollouts.}
\label{tab:attn_mismatch}
\footnotesize
\setlength{\tabcolsep}{4pt}
\resizebox{\textwidth}{!}{%
\begin{tabular}{l c ccc c c}
\toprule
 & & \multicolumn{3}{c}{$\Delta_{\mathrm{attn}}$ ($\downarrow$), forward / return} & Attention share & DINOv2 ($\uparrow$) \\
\cmidrule(lr){3-5}
Writer & $M$ ($\numChunks{=}8/16/32$) & $\numChunks{=}8$ & $\numChunks{=}16$ & $\numChunks{=}32$ & reference / slot & $\numChunks{=}8/16/32$ \\
\midrule
\worldtrace                                     & 7 / 13 / 25  & .08 / .18 & .14 / .27 & .21 / .37 & .41 / .04 & .80 / .56 / .42 \\
Full-history average                          & 19 / 43 / 91 & .19 / .36 & .28 / .45 & .37 / .54 & .57 / .03 & .81 / .58 / .44 \\
\bottomrule
\end{tabular}%
}
\end{table}

\section{The \loopbench Memory Benchmark}
\label{app:loopbench}

\begin{figure}[htbp]
\centering
\definecolor{loopboxblue}{HTML}{A9B8E0}
\definecolor{looptabblue}{HTML}{CBD7F2}
\begin{tikzpicture}[>=Stealth, thick, font=\scriptsize\rmfamily,
  nd/.style={draw, circle, fill=gray!20, minimum size=0.50cm, inner sep=0pt},
  loopout/.style={->, blue!70},
  loopret/.style={->, blue!70, dashed},
  loopoutlight/.style={->, blue!55},
  loopretlight/.style={->, blue!55, dotted},
  loopskip/.style={->, blue!55, densely dotted},
  loopoutthick/.style={->, blue!70, thick},
  loopretthick/.style={->, blue!70, dashed, thick},
  rowbox/.style={draw=loopboxblue, line width=1.2pt, rounded corners=10pt},
  rowtab/.style={fill=looptabblue, rounded corners=3.5pt, inner xsep=8pt,
                 inner ysep=3.5pt, font=\footnotesize\rmfamily, text=black}]

\def\DX{4.7}   %
\def\BoxL{-0.90}   %
\def\BoxR{12.45}   %
\def\BoxB{-1.30}   %
\def\TabX{5.775}   %
\def\RowIIY{-5.30}    %
\def\RowIIIY{-9.50}   %
\def\RowIVY{-14.55}   %

\begin{scope}[yshift=0cm]
  \draw[rowbox] (\BoxL,\BoxB) rectangle (\BoxR, 2.90);
  \node[rowtab] at (\TabX, 2.90) {Row~1\enspace varying topology ($\loopK$ intermediate waypoints)};
\end{scope}
\begin{scope}[yshift=\RowIIY cm]
  \draw[rowbox] (\BoxL,\BoxB) rectangle (\BoxR, 3.15);
  \node[rowtab] at (\TabX, 3.15) {Row~2\enspace varying rollout length ($\numChunks$ AR chunks in total, ABA topology)};
\end{scope}
\begin{scope}[yshift=\RowIIIY cm]
  \draw[rowbox] (\BoxL,\BoxB) rectangle (\BoxR, 2.05);
  \node[rowtab] at (\TabX, 2.05) {Row~3\enspace camera orientation ($90^\circ$/$180^\circ$ reverse, $360^\circ$ full turn)};
\end{scope}
\begin{scope}[yshift=\RowIVY cm]
  \draw[rowbox] (\BoxL,\BoxB) rectangle (\BoxR, 2.90);
  \node[rowtab] at (\TabX, 2.90) {Row~4\enspace multi-revisit patterns ($\loopR > 1$, solid/lighter arcs = 1st/2nd traversal)};
\end{scope}

\begin{scope}[xshift=0cm, yshift=0cm]
  \node[nd] (a) at (0.8,0)   {A};
  \node[nd] (b) at (0.8,2.1) {B};
  \draw[loopout, bend left=24] (a) to node[right]{\scriptsize $\times\!8$} (b);
  \draw[loopret, bend left=24] (b) to (a);
  \node[anchor=north, font=\scriptsize\rmfamily] at (0.8,-0.40)
    {\textbf{ABA}\quad $\numChunks{=}16$};
\end{scope}

\begin{scope}[xshift=\DX cm, yshift=0cm]
  \node[nd] (a) at (0,0)    {A};
  \node[nd] (b) at (0,1.90) {B};
  \node[nd] (c) at (1.85,1.90) {C};
  \draw[loopout] (a) -- node[left,  xshift=1pt ]{\scriptsize $\times\!5$} (b);
  \draw[loopout] (b) -- node[above, yshift=-1pt]{\scriptsize $\times\!5$} (c);
  \draw[loopret] (c) -- node[right, xshift=-1pt]{\scriptsize $\times\!7$} (a);
  \node[anchor=north, font=\scriptsize\rmfamily] at (0.92,-0.40)
    {\textbf{ABCA}\quad $\numChunks{=}17$};
\end{scope}

\begin{scope}[xshift=2*\DX cm, yshift=0cm]
  \node[nd] (a) at (0,0)    {A};
  \node[nd] (b) at (0,1.90) {B};
  \node[nd] (c) at (1.90,1.90) {C};
  \node[nd] (d) at (1.90,0)    {D};
  \draw[loopout] (a) -- node[left ]{\scriptsize $\times\!4$} (b);
  \draw[loopout] (b) -- node[above]{\scriptsize $\times\!4$} (c);
  \draw[loopout] (c) -- node[right]{\scriptsize $\times\!4$} (d);
  \draw[loopret] (d) -- node[below]{\scriptsize $\times\!4$} (a);
  \node[anchor=north, font=\scriptsize\rmfamily] at (0.95,-0.40)
    {\textbf{ABCDA}\quad $\numChunks{=}16$};
\end{scope}

\begin{scope}[xshift=0cm, yshift=\RowIIY cm]
  \node[nd] (a) at (0.8,0)   {A};
  \node[nd] (b) at (0.8,1.30) {B};
  \draw[loopout, bend left=24] (a) to node[right]{\scriptsize $\times\!4$} (b);
  \draw[loopret, bend left=24] (b) to (a);
  \node[anchor=north, font=\scriptsize\rmfamily] at (0.8,-0.40)
    {\textit{short}\quad $\numChunks{=}8$};
\end{scope}

\begin{scope}[xshift=\DX cm, yshift=\RowIIY cm]
  \node[nd] (a) at (0.8,0)   {A};
  \node[nd] (b) at (0.8,2.0) {B};
  \draw[loopout, bend left=24] (a) to node[right]{\scriptsize $\times\!8$} (b);
  \draw[loopret, bend left=24] (b) to (a);
  \node[anchor=north, font=\scriptsize\rmfamily] at (0.8,-0.40)
    {\textit{medium}\quad $\numChunks{=}16$};
\end{scope}

\begin{scope}[xshift=2*\DX cm, yshift=\RowIIY cm]
  \node[nd] (a) at (0.8,0)   {A};
  \node[nd] (b) at (0.8,2.40) {B};
  \draw[loopout, bend left=24] (a) to node[right]{\scriptsize $\times\!16$} (b);
  \draw[loopret, bend left=24] (b) to (a);
  \node[anchor=north, font=\scriptsize\rmfamily] at (0.8,-0.40)
    {\textit{long}\quad $\numChunks{=}32$};
\end{scope}

\begin{scope}[xshift=0cm, yshift=\RowIIIY cm]
  \node[nd] at (0.92, 0.70) {A};
  \draw[gray!60, semithick] (0.92, 1.28) -- (0.92, 1.45);
  \draw[loopoutthick]
    plot[domain=90:0, samples=20, variable=\t]
      ({0.92 + 0.58*cos(\t)}, {0.70 + 0.58*sin(\t)});
  \draw[loopretthick]
    plot[domain=0:90, samples=20, variable=\t]
      ({0.92 + 0.75*cos(\t)}, {0.70 + 0.75*sin(\t)});
  \node[anchor=north, font=\scriptsize\rmfamily] at (0.92, -0.40) {\textbf{90°} yaw pan};
\end{scope}

\begin{scope}[xshift=\DX cm, yshift=\RowIIIY cm]
  \node[nd] at (0.92, 0.70) {A};
  \draw[gray!60, semithick] (0.92, 1.28) -- (0.92, 1.45);
  \draw[loopoutthick]
    plot[domain=90:-90, samples=30, variable=\t]
      ({0.92 + 0.58*cos(\t)}, {0.70 + 0.58*sin(\t)});
  \draw[loopretthick]
    plot[domain=-90:90, samples=30, variable=\t]
      ({0.92 + 0.75*cos(\t)}, {0.70 + 0.75*sin(\t)});
  \node[anchor=north, font=\scriptsize\rmfamily] at (0.92, -0.40) {\textbf{180°} look-back};
\end{scope}

\begin{scope}[xshift=2*\DX cm, yshift=\RowIIIY cm]
  \node[nd] at (0.92, 0.70) {A};
  \draw[gray!60, semithick] (0.92, 1.28) -- (0.92, 1.40);
  \draw[loopoutthick]
    plot[domain=90:-268, samples=60, variable=\t]
      ({0.92 + 0.65*cos(\t)}, {0.70 + 0.65*sin(\t)});
  \node[anchor=north, font=\scriptsize\rmfamily] at (0.92, -0.40) {\textbf{360°} full pan};
\end{scope}

\begin{scope}[xshift=0cm, yshift=\RowIVY cm]
  \node[nd] (a) at (0.8,0)   {A};
  \node[nd] (b) at (0.8,2.1) {B};
  \draw[loopout, bend left=40] (a) to (b);
  \draw[loopret, bend left=40] (b) to (a);
  \draw[loopoutlight, bend left=14] (a) to (b);
  \draw[loopretlight, bend left=14] (b) to (a);
  \node[anchor=north, font=\scriptsize\rmfamily] at (0.8,-0.40)
    {\textbf{ABABA}\quad $\loopR{=}2$,\ 4~edges};
\end{scope}

\begin{scope}[xshift=\DX cm, yshift=\RowIVY cm]
  \node[nd] (a) at (0,0)    {A};
  \node[nd] (b) at (0,1.90) {B};
  \node[nd] (c) at (1.85,1.90) {C};
  \draw[loopout] (a) -- (b);
  \draw[loopout] (b) -- (c);
  \draw[loopret] (c) -- (b);
  \draw[loopret] (b) -- (a);
  \node[anchor=north, font=\scriptsize\rmfamily] at (0.92,-0.40)
    {\textbf{ABCBA}\quad $\loopR{=}2$,\ 4~edges};
\end{scope}

\begin{scope}[xshift=2*\DX cm, yshift=\RowIVY cm]
  \node[nd] (a) at (0,0)    {A};
  \node[nd] (b) at (0,1.90) {B};
  \node[nd] (c) at (1.90,1.90) {C};
  \node[nd] (d) at (1.90,0)    {D};
  \draw[loopout] (a) -- (b);
  \draw[loopout] (b) -- (c);
  \draw[loopout] (c) -- (d);
  \draw[loopskip] (d) to[bend right=28]
    node[right, xshift=-1pt, font=\scriptsize\rmfamily]{\textit{skip}} (b);
  \draw[loopret] (b) -- (a);
  \node[anchor=north, font=\scriptsize\rmfamily] at (0.95,-0.40)
    {\textbf{ABCDBA}\quad $\loopR{=}2$,\ 5~edges};
\end{scope}

\end{tikzpicture}
\caption{\textbf{\loopbench benchmark gallery.} Solid and dashed blue arrows mark the outbound and return legs, lighter arcs mark a second traversal, and dotted arrows mark shortcuts that skip intermediate waypoints. \textbf{Row~1} shows topology (ABA, ABCA, ABCDA). \textbf{Row~2} varies rollout length on the ABA topology, increasing the interval between departure and final return. \textbf{Row~3} varies camera orientation at A. The $90^\circ$ and $180^\circ$ pans reverse to the initial orientation, whereas the $360^\circ$ condition completes one revolution. \textbf{Row~4} shows multi-revisit patterns, where $\loopR$ counts returns to A in ABABA, and visits to B in the ABCBA palindrome and the ABCDBA shortcut. These pattern-specific counts are all two.
}
\label{fig:loopbench_gallery}
\end{figure}

\subsection{Task Design and Evaluation}
\label{app:loopbench_design}

We propose \textbf{\loopbench}, a benchmark for episodic memory in autoregressive world models. Each task is a navigation path that departs from a location, visits other locations, and returns. The frames generated on the first visit serve as the reference, and we measure how closely the frames generated on the return reproduce them. \loopbench therefore does not need external reference video, as the rollout itself produces both the target and the prediction.

\subsection{Difficulty Axes}
\label{app:loopbench_axes}

Four properties of the path control task difficulty.
(1)~\textbf{Waypoint count} ($\loopK$) measures the number of intermediate locations between departure and return to A ($\loopK{=}1$ for ABA and $\loopK{=}3$ for ABCDA).
(2)~\textbf{Rollout length} ($\numChunks$) measures the total number of AR chunks in the path. Longer rollouts put more chunks between the first visit and the return, so the revisited scene sits deeper in the compressed cache.
(3)~\textbf{Camera orientation} varies while the agent stays at A. The $90^\circ$ and $180^\circ$ conditions pan away and then reverse to the initial orientation. The $360^\circ$ condition makes one complete revolution.
(4)~\textbf{Multi-revisit patterns} include ABABA, which returns to A twice, while ABCBA and ABCDBA revisit B before their final return to A. The pattern-specific count $\loopR$ in the gallery is not a common depth measure across these paths. The reported metric compares the frames generated at A on the first visit and on the final return.
These tiers are controlled families of tasks rather than independent single-factor interventions. Topology changes path geometry and can change duration, and the pan conditions also use different durations. The families are intended to distinguish memory behavior across navigation settings~\citep{lorraine2022taskselection}. We show all twelve configurations in \Fig{fig:loopbench_gallery} and evaluate every one of them in \Tab{tab:loopbench}.
We show qualitative reconstructions on the ABA (axis~1) and pan~$90^\circ$ (axis~3) settings in \Fig{fig:app_aba_qualitative} and \Fig{fig:app_pan90_qualitative}.

\begin{figure}[t]
  \centering
  \includegraphics[width=\linewidth]{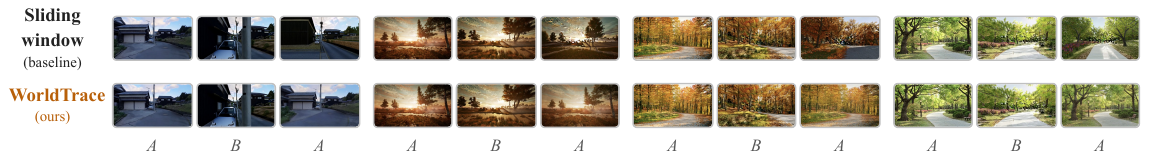}
  \caption{\textbf{ABA qualitative results.} Each sample group of three frames shows trajectory keyframes at chunks 0~(A) / 7~(B) / 15~(A return). The two rows compare the sliding-window baseline (top) vs.\ \worldtrace (bottom). The sliding window drifts away from the scene-A appearance on the return leg across the four samples. \worldtrace brings the return frames back to the original scene appearance.
  }
  \label{fig:app_aba_qualitative}
\end{figure}

\begin{figure}[t]
  \centering
  \includegraphics[width=\linewidth]{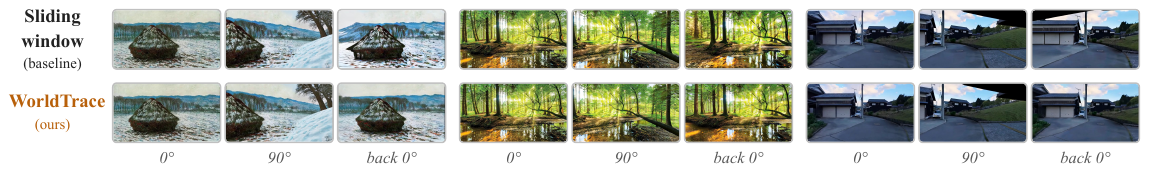}
  \caption{\textbf{Pan~$90^\circ$ qualitative results.} In camera-orientation Tier~3, the agent stays at A while the camera pans right by ${\sim}90^\circ$ and then pans back. Each sample group of three frames shows the start ($0^\circ$), the pan peak (${\sim}90^\circ$), and the return to $0^\circ$. The two rows compare the sliding-window baseline (top) and \worldtrace (bottom). \worldtrace restores the initial-view appearance on the return. The sliding window does not.
  }
  \label{fig:app_pan90_qualitative}
\end{figure}

\section{Implementation Details}
\label{app:impl}

\subsection{Model Architecture and Self-Forcing Background}
\label{app:background}

We evaluate on MG2-1.3B~\citep{matrixgame2025}, a 1.3B-parameter autoregressive video world model built on Wan 1.3B T2V~\citep{wan2025} with 30 layers, 12 attention heads, and head dimension 128. It generates video in chunks of 3 latent frames at $44\times80$ latent resolution ($352\times640$ pixels), or 880 tokens per latent frame, and denoises each chunk with a distilled flow-matching sampler conditioned on the KV cache of prior chunks. Self-Forcing training~\citep{selfforcing2025} rolls the model out on its own KV cache rather than on teacher-forced context. Attention is restricted to a local window of $\localAttnSize{=}6$ latent frames, so temporal offsets stay $\leq\trainOffset{=}5$. \worldtrace keeps the cache at this fixed budget, about $0.97$\,GB in fp16 at batch size~1, independent of generation length.

\subsection{\worldtrace Algorithm }
\label{app:worldtrace_impl}

After each new chunk is generated, its keys and values enter the verbatim recent window. When the recent window exceeds $\numRecentSlots$ chunks, the oldest chunk is evicted into the summary cache as a new canonical slot. If the summary cache then exceeds $\numSummarySlots$ slots, attention-response error on sampled queries selects which slots to merge or evict, and key-value similarity decides whether they are averaged into a field summary or kept as a landmark (\Sec{subsec:combined}).
We list these update steps in \Alg{alg:update}, together with the final step that re-rotates each canonical summary key to its virtual position when attention scores are computed.

\begin{algorithm}[t]
\caption{\worldtrace cache update. $\summaryCache$ is one bounded bank of canonical slots with counts $c$. $\mathrm{Attn}(\queryVec;\cdot)$ is the attention output of query $\queryVec$ over a cache, and $\mathcal{E}(u)$ is the attention-response error of update $u$, the change in the sampled queries' attention outputs plus a penalty for losing unique information. Each step applies the update that minimizes this error, so compression alters what the model retrieves as little as possible.}
\label{alg:update}
\begin{algorithmic}[1]
\Require recent cache $\recentCache$, summary cache $\summaryCache$, new chunk's KVs, sampled queries $\querySet$
\Ensure updated $(\recentCache, \summaryCache)$
\State $\recentCache \gets \recentCache \cup \{\text{new KVs}\}$ \Comment{recent window is verbatim}
\If{$|\recentCache| > \numRecentSlots$}
  \State Pop the oldest $(K, V)$ at time $t$ from $\recentCache$
  \State $\summaryCache \gets \summaryCache \cup \{(\Rot{-\ropefreq t}K,\, V,\, c{=}1)\}$ \Comment{canonical key}
  \If{$|\summaryCache| > \numSummarySlots$}
    \State $u^\star \gets \arg\min_u \mathcal{E}(u)$ over merges and evictions \Comment{min.\ attention-response error}
    \If{$u^\star$ merges compatible slots $i, j$}
      \State $(K_i, V_i, c_i) \gets (\bar{K}, \bar{V}, \bar{c})$, remove slot $j$ \Comment{\wtfield}
    \ElsIf{$u^\star$ merges incompatible slots $i, j$}
      \State Retain a paired exemplar \Comment{\wtlandmark}
    \Else
      \State Remove the evicted slot
    \EndIf
  \EndIf
\EndIf
\State At attention time, $K_s \gets \Rot{\ropefreq \tvslot}\,\summaryCache[\slotIdx].K$ for $\slotIdx{=}0,\ldots,\numSummarySlots{-}1$ \Comment{\Def{def:vpos}}
\end{algorithmic}
\end{algorithm}

\subsection{Hyperparameters, Compute, and Evaluation Protocol}
\label{app:hyperparams}

\begin{wraptable}{r}{0.42\linewidth}
\vspace{-5pt}
\centering
\caption{\textbf{Runtime per chunk.} Wall-clock time at $\numChunks{=}8$ with VAE decode, on one L40 48\,GB (batch~1, 3 distilled denoising steps, $352{\times}640$), mean over 51 rollouts.}
\label{tab:runtime}
\small
\begin{tabular}{lr}
\toprule
Method & s/chunk \\
\midrule
Sliding window & 1.10 \\
\wtfield    & 1.14 \\
\wtlandmark & 1.20 \\
\wtunified  & 1.40 \\
\bottomrule
\end{tabular}
\end{wraptable}
\paragraph{Inference hyperparameters.}
MG2-1.3B~\citep{matrixgame2025} uses 3 distilled denoising steps ($t \in \{1000, 666, 333\}$ with timestep shift $5.0$), no classifier-free guidance at inference, a single conditioning image, and $\framesPerBlock{=}3$ latent frames per AR chunk. Each chunk is additionally conditioned on per-frame keyboard and mouse actions through the model's action module, which is how the camera trajectories are controlled. In our setting, the first chunk is generated with an empty KV cache, and every later chunk attends to the accumulated compressed cache. The MG2 \loopbench and main ablation experiments use $\numSummarySlots{=}4$ and $\numRecentSlots{=}2$.

\paragraph{Compute.}
We report wall-clock time per AR chunk (including VAE decode) on one L40 48\,GB at batch size~1 in \Tab{tab:runtime}. At the horizon $\numChunks{=}8$, \wtfield and \wtlandmark stay within 4\% and 9\% of the sliding-window baseline, since the forward pass dominates the cache update. \worldtrace adds 27\% and the recompute writer revisits a growing history.

\section{Discussion}
\label{app:discussion}

\subsection{Limitations}
\label{app:limitations}

\worldtrace operates entirely at inference time, leaving the pretrained weights frozen and changing only what information the fixed-size KV cache retains (compressed summaries in place of evicted frames) and the virtual RoPE positions at which those entries are read, with the goal of improving long-horizon memory.
Everything else, including per-chunk visual quality, motion dynamics, and action following, is inherited unchanged from the pretrained backbone. The method also assumes an autoregressive generator with temporal RoPE applied to keys and a known local attention window.
Compression into a fixed budget of $\numSummarySlots$ slots is lossy by construction, since each \wtfield slot averages roughly $\numPastFrames/\numSummarySlots$ frames, a ratio that grows with horizon, so frame-level detail inside a group is progressively blurred. Under \wtlandmark, a scene can be recalled later only if the scene-entry detector stored it as a landmark.
Moreover, the cache holds only $\numSummarySlots$ landmarks, so a rollout that enters more scenes than that loses the oldest ones and can reliably revisit at most $\numSummarySlots$ distinct places. Future work could relax this cap, for example by merging evicted landmarks into a coarse residual summary slot so that discarded scenes leave a recoverable trace, or by selecting which landmark to evict based on predicted revisit likelihood.

\subsection{Future Directions}
\label{app:future_work}

\paragraph{Geometry-aware canonical keys.}
Coupling the unrotate/rerotate primitive with camera-pose warping (\eg the Pl\"{u}cker and Warped-RoPE writers of MosaicMem~\citep{yu2026mosaicmem} and UCM~\citep{xu2026ucm}) would align frames to a shared scene coordinate system and then average, so each summary token blends observations of similar scene content.

\paragraph{Learned scene-entry policies.}
A small policy trained on action discontinuities, agent-pose changes, or scene-segmentation logits could let \wtlandmark commit landmarks during continuous motion and decide \emph{which} landmark to select. The fixed split between $\numSummarySlots$ and $\numRecentSlots$ could become an autotunable schedule and adapt per backbone or horizon without a manual sweep.

\paragraph{Fine-tuning extensions.}
\worldtrace stays within the attention window $\localAttnSize$, where the budget $\numSummarySlots + \numRecentSlots = \localAttnSize$ keeps every summary slot in-distribution. Two light fine-tuning paths relax this budget. (i)~\emph{Context-extension fine-tuning} on synthetic long rollouts would let \worldtrace allocate more summary slots for longer recall at the same in-distribution cost. (ii)~\emph{Position-aware fine-tuning} that trains the model on the slot-rank offsets of \Eq{eq:vpos} would tighten the canonical-mean approximation of \wtfield.

\paragraph{Training memory policies from revisit supervision.}
The memory policy that decides what to store, compress, and evict could itself be trained, and revisit tasks supply the training signal for free, since whether the model reproduces the original scene on its return is directly measurable. Closed-loop simulators can collect such revisit episodes at scale, supporting supervised fine-tuning on rollouts where recall succeeds as well as reinforcement learning with the revisit score as reward. In both cases the backbone stays frozen and \worldtrace keeps every stored summary in-distribution, so only the lightweight policy over the cache is trained.

\paragraph{Scaling \loopbench.}
We propose \loopbench (\App{app:loopbench}) with four difficulty tiers, and it can grow beyond them, since longer paths, more waypoints, and richer camera trajectories are all generated by the same recipe, and evaluating other world models on it would show how well their memory mechanisms support revisits, complementing WorldScore~\citep{duan2025worldscore}, MIND~\citep{ye2026mind}, and VBench-2.0~\citep{zheng2025vbench2}.

\subsection{Broader Impact}
\label{app:broader_impact}

\worldtrace is an inference-time cache mechanism that alters no training data, modifies no weights, and expands none of the generative capabilities of the underlying video model, so the relevant considerations are those already attached to pretrained video world models. By keeping long-horizon generation within a constant memory budget, it also lowers the compute cost of building on pretrained generators, a cost that shapes the wider ecosystem reusing them, from score-distilled 3D synthesis~\citep{lorraine2023att3d,xie2024latte3d} with compute-aware gradient estimators~\citep{bettencourt2026carv,richterpowell2025audiosds} to LLM-conditioned mesh generation~\citep{wang2024llamamesh} and motion attribution for video models~\citep{wu2026motive}. We release no new datasets, and downstream use should follow the base models' content policies.

\end{document}